\documentclass{article}
\usepackage{ijcai23}

\usepackage{times}
\usepackage{soul}
\usepackage{url}
\usepackage[hidelinks]{hyperref}
\usepackage[utf8]{inputenc}
\usepackage[small]{caption}
\usepackage{graphicx}
\usepackage{amsthm}
\usepackage{booktabs}
\usepackage{algorithm,algpseudocode}
\usepackage[switch]{lineno}
\usepackage{xcolor}             
\usepackage{microtype}
\usepackage{adjustbox}
\usepackage[font=footnotesize]{subfig}
\usepackage{xspace}   
\usepackage{booktabs} 

\usepackage{colortbl}       
\usepackage{multirow}       
\usepackage{pifont}         
\usepackage{soul}
\usepackage{enumitem}
\usepackage{mathrsfs}
\usepackage{amsmath,amsfonts,bm}
\usepackage{stackengine}
\usepackage{diagbox}
\usepackage[author={CL}]{pdfcomment}
\usepackage{wrapfig}

\usepackage{tikz}           

\def\vc{{\bm{c}}}

\def\vq{{\bm{q}}}

\def\vC{{\bm{C}}}
\def\vE{{\bm{E}}}
\def\vK{{\bm{K}}}
\def\vP{{\bm{P}}}
\def\vQ{{\bm{Q}}}
\def\vR{{\bm{R}}}

\def\vV{{\bm{V}}}
\def\vW{{\bm{W}}}

\newcommand{\bull}{{\tiny $\bullet~~$ }}
\newcommand{\myul}[2][black]{\setulcolor{#1}\ul{#2}\setulcolor{black}}

\usepackage{amsmath,amsthm,amssymb,wasysym,bussproofs,csquotes}

\newcommand{\upps}{\textsl{\textsc{ProPS}}\xspace}
\newcommand{\adapter}{\textsl{\textsc{adapter}}\xspace}
\newcommand{\camtl}{\textsl{\textsc{ca-mtl}}\xspace}
\newcommand{\prefix}{\textsl{\textsc{prefix}}\xspace}
\newcommand{\prefixplus}{\textsl{\textsc{prefix++}}\xspace}
\newcommand{\trsfp}{\textsl{\textsc{trsf-p}}\xspace}
\newcommand{\utrsfp}{\textsl{\textsc{s-trsf-p}}\xspace}
\newcommand{\bart}{\textsl{\textsc{bart}}\xspace}
\newcommand{\mbart}{m\textsl{\textsc{bart}}\xspace}
\newcommand{\tfive}{\textsl{\textsc{T5}}\xspace}
\newcommand{\highlight}[1]{\colorbox{blue!10}{#1}}
\definecolor{mygray}{gray}{0.4}
\newcommand{\g}[2]{#1\textsubscript{\textcolor{mygray}{$\pm$#2}}}

\RequirePackage{etoolbox}
\RequirePackage{pgf} 
\RequirePackage{xcolor} 
\definecolor{high}{HTML}{F0FFFF}  
\definecolor{low}{HTML}{007FFF}  
\newcommand*{\opacity}{100}
\newcommand*{\minval}{0.0}
\newcommand*{\midval}{0.7}
\newcommand*{\maxval}{1.0}
\newcommand{\gradient}[1]{
    \ifdimcomp{#1pt}{<}{\midval pt}{
        \pgfmathparse{int(round(100*(abs(#1-\midval)/(\midval-\minval))))}
        \xdef\tempa{\pgfmathresult}
        \hspace{-0.6em}\cellcolor{high!\tempa!low!\opacity} #1 \hspace{-0.6em}
    }
    {
        \pgfmathparse{int(round(100*(abs(#1-\midval)/(\maxval-\midval))))}
        \xdef\tempa{\pgfmathresult}
        \hspace{-0.6em}\cellcolor{high!\tempa!low!\opacity} #1 \hspace{-0.6em}
    }
}
\newcommand*{\minvalr}{23.0}
\newcommand*{\maxvalr}{30.7}
\newcommand{\gradientresult}[1]{
    \pgfmathparse{int(round(100*(abs(#1-\minvalr)/(\maxvalr-\minvalr))))}
    \xdef\tempa{\pgfmathresult}
    \hspace{-0.6em}\cellcolor{low!\tempa!high!\opacity} #1 \hspace{-0.6em}
}

\theoremstyle{plain}
\newtheorem{theorem}{Theorem}[section]
\newtheorem{example}[theorem]{Example}
\newtheorem{proposition}[theorem]{Proposition}
\newtheorem{lemma}[theorem]{Lemma}
\newtheorem{corollary}[theorem]{Corollary}
\theoremstyle{definition}

\theoremstyle{remark}
\newtheorem{remark}[theorem]{Remark}

\newcommand{\mycomment}[3]{}

\newcommand*\circled[1]{\tikz[baseline=(char.base)]{
            \node[shape=circle,draw,inner sep=0.5pt] (char) {#1};}}

\title{On Conditional and Compositional Language Model Differentiable Prompting}

\author{
  Jonathan Pilault$^1$\thanks{Work completed during an internship at Amazon Research.}, Can Liu$^2$, Mohit Bansal$^3$, Markus Dreyer$^2$ \\
  \affiliations
  $^1$Mila - Québec AI Institute, Polytechnique Montréal\\
  $^2$Amazon Alexa, \\
  $^3$University of North Carolina at Chapel Hill \\
}

\begin{document}

\maketitle

\begin{abstract}
Prompts have been shown to be an effective method to adapt a frozen Pretrained Language Model (PLM) to perform well on downstream tasks. Prompts can be represented by a human-engineered word sequence or by a learned continuous embedding. 
In this work, we investigate conditional and compositional differentiable prompting.
We propose a new model, Prompt Production System (\upps), which learns to \emph{transform} task instructions or input metadata, into continuous prompts that elicit task-specific outputs from the PLM. 
Our model uses a modular network structure based on our neural formulation of Production Systems, which allows the model to learn discrete \emph{rules} -- neural functions that learn to specialize in transforming particular prompt input patterns, making it suitable for compositional transfer learning and few-shot learning. 
We present extensive empirical and theoretical analysis and show that \upps consistently surpasses other PLM adaptation techniques, and often improves upon fully fine-tuned models, on compositional generalization tasks, controllable summarization and multilingual translation, while needing fewer trainable parameters.
\end{abstract}

\section{Introduction}

\label{sec:intro}

Humans have a remarkable ability to solve cognitive tasks by following instructions and using prior relevant experiences. 
Can a machine do the same? Recent research has shown that textual task instructions \cite{engineering_instructions} appended to a Pretrained Language Model's (PLM) input can yield successful results in several NLP tasks such as classification 
\cite{schick-schutze-2021-exploiting}, image captioning \cite{multimodal-prompt} and question-answering \cite{nat_instruction_qa} \emph{without fine-tuning the full model}. 
However, PLMs do not always accurately grasp the meaning of textual prompts \cite{nat_instruction_qa} and often display high sensitivity to the word formulation \cite{prompt-sensitivity} of instructions. 
Further, unless the model exceeds billions of parameters, full-model fine-tuning still typically outperforms human-engineered prompts for frozen PLMs \cite{lester2021power}. 
To mitigate such issues, differentiable and continuous prompting techniques \cite{prefix} have emerged as a viable alternative (illustrated in Figure~\ref{subfig:prepend}). 
However, current techniques optimize prompts solely based on the task learning objective without leveraging information embedded in textual prompts.

\begin{figure}[ht!]
\centering
\begin{tabular}{c}
\hspace{-5mm}
\adjustbox{valign=b}{
    \subfloat[Prepending generated differentiable prompts.\label{subfig:prepend}]{%
        \includegraphics[width=.95\linewidth]{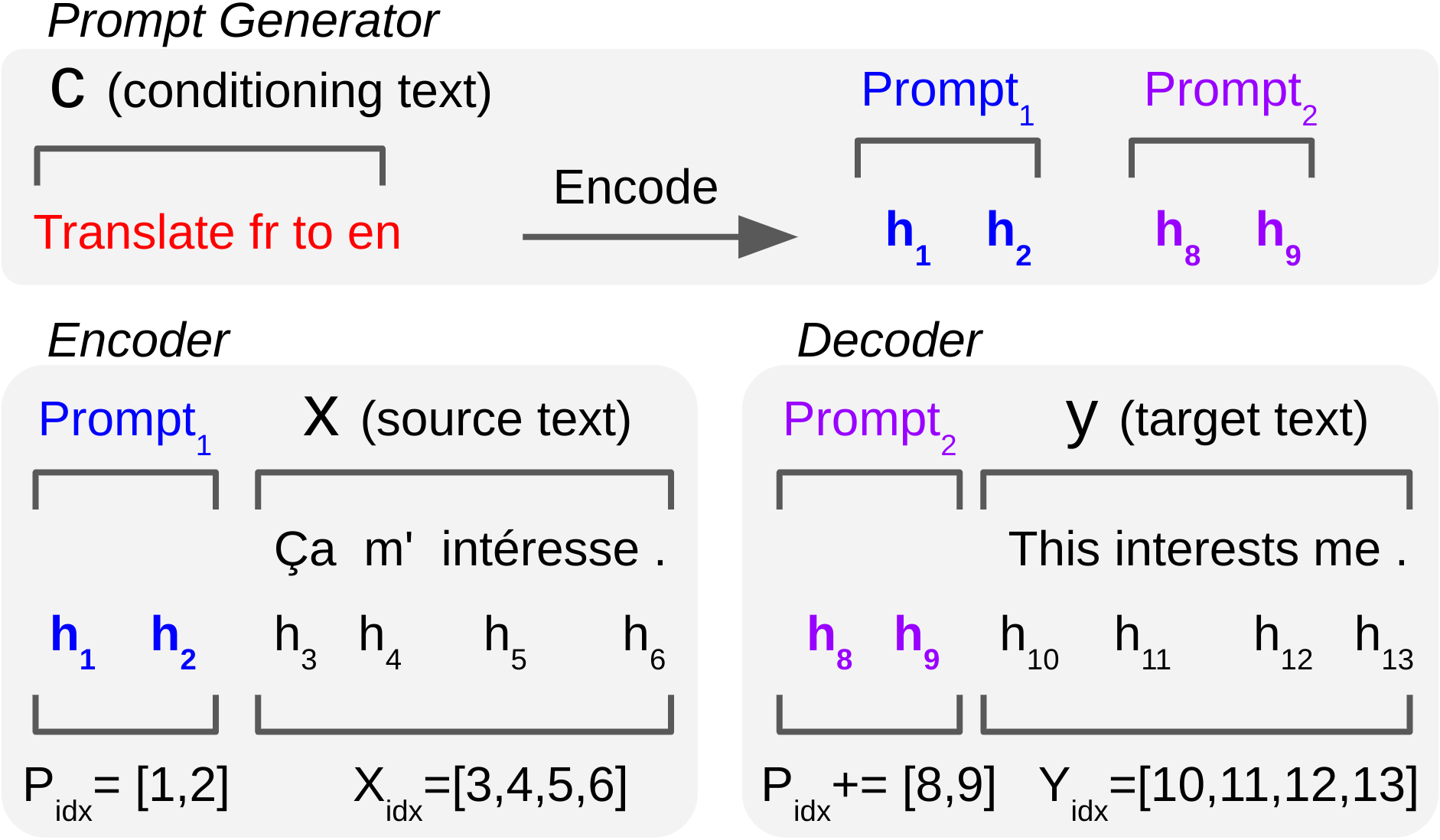}} 
}
\end{tabular}
\caption{\small \upps is a differentiable and conditional prompt generator that outputs a sequence of vectors that is prepended to PLM hidden states at positions $\in P_{\text{idx}}$
as seen in (a). 
}\label{fig:method}
\end{figure}

We argue that simple task instructions (e.g., "Translate English to French"), along with other already available or easily extractable metadata, such as the text category of an input sample (e.g., "sports"), can be valuable conditioning context for a model to \emph{generate} a continuous prompt. 
The prompt would be specific to the particular task and input, and would elicit task- and input-specific output from a frozen PLM. 
We propose a model that can, given a PLM, be trained to \emph{transform} task- and input-specific text into conditional continuous prompt embeddings. 
The advantage of using text conditioning as input to a continuous prompt generator is that our model can, once trained, repurpose previously unseen instructions from previously trained tasks, without having to retrain a PLM. 
By sharing aspects in the task instructions that were observed during training of related tasks, the continuous prompts of new tasks are more suitable for transfer learning and few-shot learning. 
With our method, we seek to combine the advantages of simple task descriptions, which provide useful context, and learned continuous prompts, which work well even with smaller PLMs. 
In addition, we propose to enhance our model's compositional ability for few-shot and transfer learning by using a modular network structure, based on our own neural formulation of Production Systems (ProdSys, \cite{newell_problem_solving}). 

\begin{figure}[ht!]
\small
\centering
\includegraphics[width=0.97\linewidth, angle=0]{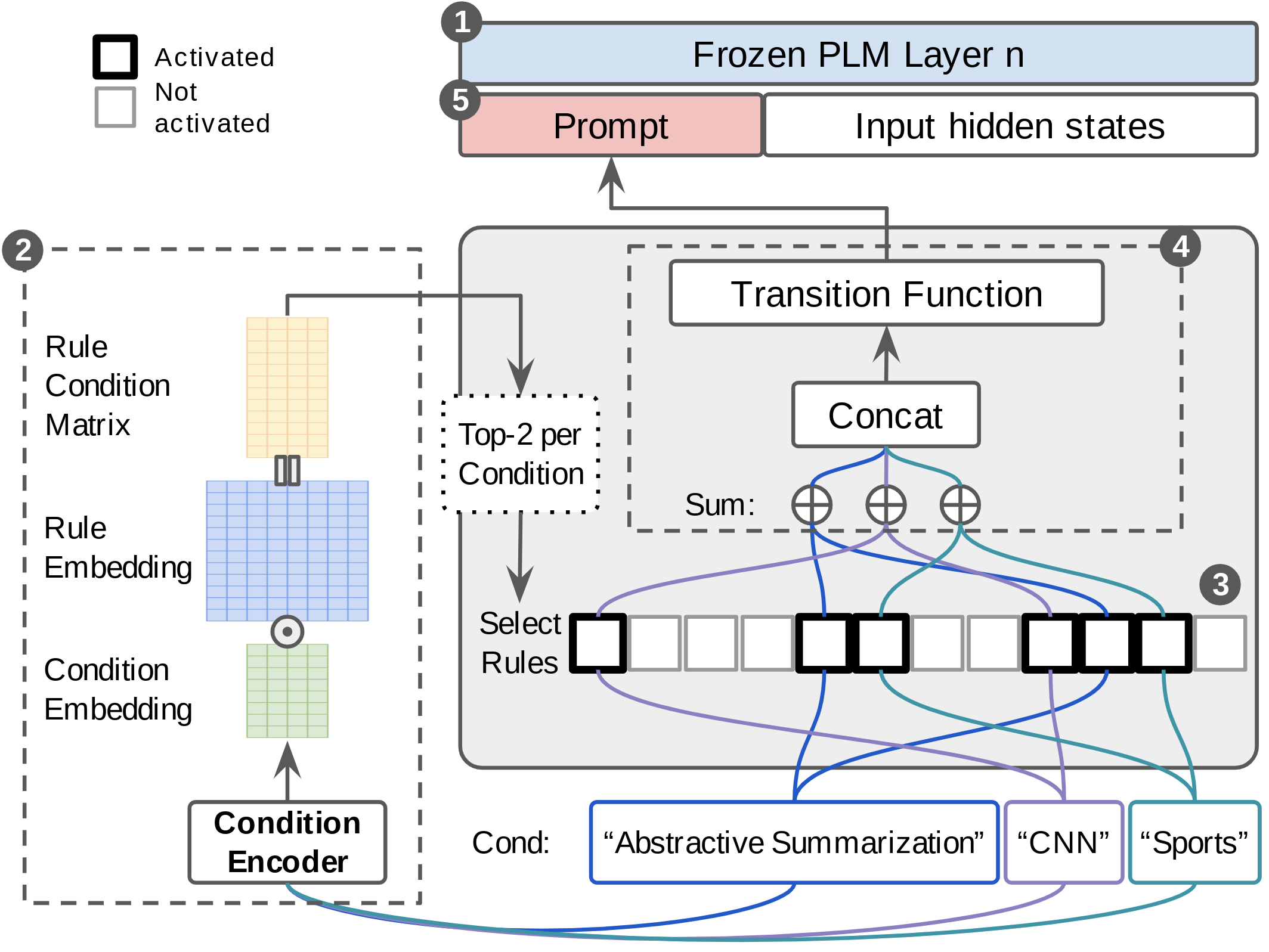}
  \caption{\upps architecture with top-$k$=2 rule selection as an example. ``Cond''= conditioning text, example inputs to \upps. \label{subfig:pps_layer}}
\end{figure}

ProdSys has several desirable properties. First, the ability to reason in a propositional format (if - then) allows the disentanglement and \emph{modularization of knowledge} about a task or input sample. 
Second, using the fact that propositions are discrete and independent from each other, the system can combine productions and \emph{compose acquired knowledge}. 
We hypothesize that such properties are desirable to improve generalization, transfer of learning, and reusability of prior task descriptions when generating prompts. 
ProdSys models symbolic knowledge in the form of ``IF'' (condition or situation or task) - ``THEN'' (action or instructions or response) rules, called productions, and a set of mechanisms for matching and applying such rules. 
Generally, ProdSys contains four basic components \cite{ai_quest}:
\circled{1} a long-term memory, a knowledge base that contains general information independent of the current situation/task and which persists over time and that may be used as a condition;
\circled{2} a working memory that contains temporary information about the current situation or task and that also operates as a sparse rule selector; 
\circled{3} a rule memory that contains a collection of productions;
\circled{4} an inference engine that maps a condition $C_j$ to an instruction $I_i$ in the form $C_j \xrightarrow[]{R_i} I_j$, for selected productions $R_i$. 

\begin{figure*}[ht]
\centerline{
\includegraphics[trim={0 0 0 0}, width=0.90\textwidth]{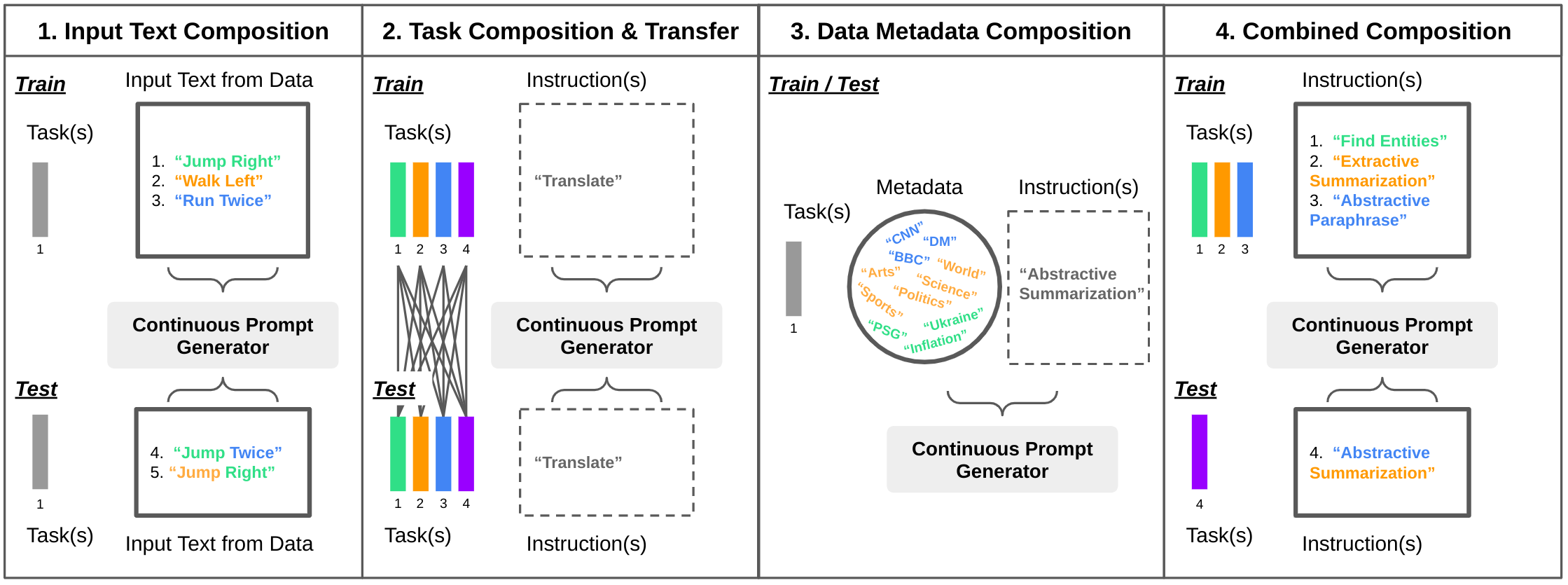}}
\caption{
\small Outline of the capabilities investigated by training composable inputs, tasks and/or metadata and tested on different combinations. 
The \textbf{colored} artifacts in the solid line circles or squares represent information that can be composed for controlled text generation.
}
\label{fig:evaluations}
\end{figure*}

We draw a parallel between the process of generating continuous prompts (differentiable instructions for the PLM) from textual conditioning (conditions) and rule productions. 
Specifically, as shown in Fig.~\ref{subfig:pps_layer}: 
a frozen PLM can be used as a persistent knowledge base~\cite{plms_as_kg} 
or \circled{1} a long-term memory\footnote{The neural long-term memory is used in two ways: the conditions can be hidden states from the previous PLM layer and the gradient updates are dependent of the PLM pretrained weights.}; 
a Condition Encoder that maps textual conditioning sequences to a fixed-sized vector to create the Rule-Condition matrix works similarly to \circled{2} a working memory by encoding and sparsely matching conditions and rules; 
the selected differentiable rules, the summation operation and transition function allows the model to map a condition to differentiable instructions and work as \circled{3} a rule memory and \circled{4} an inference engine; finally, instructions \circled{5} are, in our case, differentiable prompts that we prepend to a frozen PLM's hidden states.
Such a system, which we call Prompt Production System (\upps), shares the four main building blocks of ProdSys, which allows it to conditionally generate prompts while compartmentalizing neural propositional knowledge via its condition-rule selection mechanism.
As we will discover, such mechanism equips \upps with an ability to reuse and re-purpose knowledge from examples of the same tasks or instructions of intermediate and jointly trained tasks that allow the model to better control text generation.
This construction allows \upps to surpass all other prompt generator baselines, and often fully fine-tuned models, on summarization, translation, and semantic parsing benchmarks.

Our contributions are the following:
\begin{itemize}[noitemsep,partopsep=0pt,topsep=0pt,parsep=0pt,leftmargin=4mm]
    \item We provide strong Conditional Prompt Generator baselines (Section~\ref{sec:cpg}) on multiple tasks and datasets, showing the advantage of conditioning continuous prompts.
    \item We propose a new architecture, called \upps, a prompt generator based on ProdSys (Section~\ref{sec:upps}).
     \item We demonstrate \upps’s ability to compose knowledge in the four settings described above in both the full-data regime and low-resource settings.
    \item To better understand \upps, we provide theoretical and empirical analysis in Appendix~\ref{sec:analysis}.
\end{itemize}

We evaluate our methods in four situations as illustrated in Fig.~\ref{fig:evaluations} to check if the following can be composed: 
(1) input text from training set input text segments in Section~\ref{sec:input_text},
(2) tasks when knowledge is transferred between all tasks during training in Section~\ref{sec:multitask},
(3) metadata about the input data in Section~\ref{sec:meta-info}, 
(4) all of the above combined in Section~\ref{sec:composition}.

\section{Related Works}
\label{sec:lit}

\textbf{Differentiable Prompting} has seen increased interest lately due to the method's ease of use and compelling performance. 
Unlike prior work, \upps is a conditional prompt generator.
There are two differentiable prompting methods that are close to \upps. 
The first is PTR \cite{prompt_tuning_rule_2021} that breaks an input into a series of predicates of first order logic.
However, unlike \upps, sub-prompts are manually designed and PTR is mainly useful for classification tasks. 
Concurrently to our work, PGT \cite{prompt_transfer_nlg_2022} also attempts to transfer prompts learned from multiple representative support text generation tasks (e.g. paraphrasing, question answering). 
Instead of using an attention mechanism to select task-specific prior prompts, \upps compartmentalizes and reconstructs prior prompt knowledge using differentiable rules, allowing module sharing across tasks.
Using fewer support tasks, \upps outperforms PGT on XSum on the fully-supervised and few-shot learning settings.

\noindent\textbf{Modular Networks} are a type of neural architecture that splits computation into specialized and independent neural sub-systems (i.e., modules). 
Modular Networks borrow from conditional computation methods, where modules are
chosen dynamically given a certain criterion. Such module specialization has been shown empirically to improve generalization to unseen inputs %
\cite{NEURIPS2018_310ce61c} or tasks \cite{goyal2021recurrent}. 
In a multitask learning setting, several recent approaches have leveraged modularity to decompose knowledge \cite{ostapenko2021continual} such that it could be reused for a new task. 
Similar to \citeauthor{combine_skills_2022}, our method also assumes that tasks can be decomposed into a set of modules that are shareable and reusable across tasks. \upps differs in three ways: 
(1) we use both tasks and inputs to condition computation; 
(2) sub-module selection is informed by external information (text conditioning);
(3) we use differentiable prompts on a PLM instead of adapters.

\section{Methodology}
\label{sec:method}

We now present the two components on which our method is based. 
We will first provide an overview of Conditional Prompt Generators (CPG), linking conditional continuous prompts and adapters. 
Then, we discuss \upps, our novel construction of CPGs using Neural Production Systems.


\subsection{Conditional Prompt Generators (CPG)} 
\label{sec:cpg}
Prompt Generators such as Prefix-Tuning \cite{prefix}, Prompt-Tuning \cite{lester2021power} and P-Tuning \cite{p-tuning} are all \emph{Unconditional Prompt Generation (UPG)} methods used to adapt PLMs. 
Generally, prompt vectors $\vP_k,\vP_v \in \mathbb{R}^{T_P \times d}$, of length $T_P$ and dimension $d$, are concatenated to the key $\vK$ and value $\vV$ vectors to modulate the frozen PLM attention output of each head $H$. 
We get for a query $\vQ_t$ at step $t$:
\begin{equation}
\label{eq:prompt-adapter}
\begin{split}
& H = \text{Attn}\big(\vQ_t, (\vP_k \big\Vert \vK), (\vP_v \big\Vert \vV)\big) \\
& =  \sigma\big(\vQ_t (\vP_k \big\Vert \vK)^\top\big) (\vP_v \big\Vert \vV) \\
& = \alpha \cdot \sigma(\vQ_t\vP_k^\top)\vP_v + (1 - \alpha) \cdot \sigma(\vQ_t\vK^\top)\vV \\
& = \alpha \cdot \underbrace{ \text{Attn}(\vQ_t, \vP_k, \vP_v) }_{\text{\textcolor{orange}{Learnable task adaptation}}} + (1 - \alpha) \cdot \underbrace{ \text{Attn}(\vQ_t, \vK, \vV) }_{\text{\textcolor{blue}{Frozen PLM attention}}},
\end{split}
\end{equation}
where $\sigma$ is the softmax function and $\alpha=\alpha(\vQ_t, \vK, \vP_k)$ is a learnable gating function:
\begin{equation}
\label{eq:gate}
\alpha = \frac{\sum_i\exp (\vQ_t\vP_k^\top)_i}{\sum_i \exp (\vQ_t\vP_k^\top)_i + \sum_j \exp(\vQ_t\vK^\top)_j}.
\end{equation}

\emph{Conditional Prompt Generation}, as in \upps, is achieved by plugging condition-dependent prompt vectors $[\vP_k, \vP_v] = [\vP_k(c_t), \vP_v(c_t)] = \vP(c_t)$ in Equation \ref{eq:prompt-adapter}, where $\vP_k, \vP_v$ is created by splitting the last dimension of $\vP(\vc_t) \in \mathbb{R}^{T_P \times 2d}$ in two and $\vc_t$ is the word embedding at the $t^{\text{th}}$ condition token.

\begin{table}[h]

\begin{center}
\scriptsize
\begin{tabular}{ll|l|l}
\toprule
                   & Setting  & Conditioning text ($C$) examples & Task     \\
\midrule
\multirow{2}{*}{1} & Input        & ``\{Look, Opposite, Left, Thrice\}'' +    & \multirow{2}{*}{cg} \\
                   & Composition  & +``\{and\}''+``\{Jump, Right, Thrice\}''   &                  \\
                   \hline
\multirow{2}{*}{2} & Task  & {``Translate fr to en''+ [French Text]},  & \multirow{2}{*}{nmt} \\
                   & Composition  & {``Translate es to en''+ [Spanish Text]} &                  \\
                   \hline
\multirow{2}{*}{3} & Metadata  & ``\{cnn, dm, bbc\}'' +      & \multirow{2}{*}{sum} \\
                   & Composition &  ``\{sports, politics, business, local, ...\}'' &                 \\
                   \hline
\multirow{2}{*}{4} & Instructions,  & ``\{Find Entities, Paraphrase,    & \multirow{2}{*}{sum} \\
                   & tasks, metadata &  Summarization\}'' + [metadata] &                 \\

\bottomrule
\end{tabular}

\end{center}
\caption{\small Conditioning text inputted to the prompt generator. \textbf{cg} = compositional generalization, \textbf{nmt} = multilingual translation and \textbf{sum} = summarization.}\label{tab:conditions_eg}
\end{table}

\paragraph{Conditions $C \in \mathcal{C}$} represent posterior knowledge about a task or dataset as well as instructions. The Conditions or conditioning text are inputs to the CPGs. We provide examples in Table~\ref{tab:conditions_eg} following the different evaluation settings depicted in Fig~\ref{fig:evaluations}. 
In certain cases, it can be the input data itself as in the compositional generalization task (e.g.: input ``walk right twice and walk left'' has the output ``turn\_right walk turn\_right walk turn\_left walk'').
The conditioning text is easily extractable.
For news summarization, we extracted the news outlet and the article type.
For the topic-focused news summarization, additional inputs such as the topic is provided with the dataset.
Before going through a word embedding layer, each condition is a textual description or an instruction about an example, a type of an example, a task or a dataset. 
Instructions were written by the authors. 
See Appendix \ref{sec:app_conditions} for more details. 

\subsection{Prompt Production Systems (\upps)}
\label{sec:upps}

\upps is inspired by Neural Production Systems (NPS, \citeauthor{goyal2021neuralprodsys}), the deep learning version of ProdSys. 
NPS is an end-to-end model that constructs object- or condition-centric representations on which independent learnable production rules are applied. 
\upps also has independent learnable production rules but there are three main differences with NPS. 
First, rules are applied to sub-sequences of variable lengths instead of a single vector. 
Second, \upps is tied to a PLM through gradient updates and the appendage of prompt vectors to the input hidden states whereas NPS was not built to interact with a PLM. 
Third, our rule selection mechanism has a different design that allows rule overlap amongst conditions of a single instance. 
Unlike NPS that uses MLPs, our rules are attention heads \cite{vaswani2017_attn}. 
\upps consists of $N$ separately differentiable rules, $\{\vR_1, \vR_2,.., \vR_N\}$ that are each represented by $\vR_i = (\Vec{\vR_i}, \widetilde{\mathrm{H}}_i)$, where $\Vec{\vR_i}$ is a rule embedding vector and  $\widetilde{\mathrm{H}}_i$ is an attention head. 
As summarized in Algorithm~\ref{alg:upps}, each condition sequence $\mathcal{S}_C = \langle \vc_t | t \in \{1, \ldots,  T_C\}\rangle_C$, where $T_C$ is the sequence length of condition $C \in \mathcal{C}$ and $\mathcal{C}$ is the set of conditions, is first encoded by a Condition Encoder $f(\cdot)$. 
We use a light-weight attentive max pooling layer \cite{attentive-pooling} to encode condition sequences into a fix sized representation. 
Each condition vector $\Vec{\vC}$ is then concatenated together to create a Condition Embedding matrix $\begin{bmatrix} \vE \end{bmatrix} \in \mathbb{R}^{|\mathcal{C}| \times d}$, where $d$ is the embedding dimension. 
Similarly, the learnable Rule Embedding matrix $\begin{bmatrix} \vR \end{bmatrix}  \in \mathbb{R}^{n \times d}$, for $n$ rules (i.e., attention heads), is created by concatenating each rule representation $\Vec{\vR_i}$. 
The rule-condition matrix is formed from the dot product operation $\begin{bmatrix} \vE \end{bmatrix} \cdot \begin{bmatrix} \vR \end{bmatrix}^{\top} \in \mathbb{R}^{|\mathcal{C}| \times n}$ and provides a template that maps a condition to a set of rules. 
We use a Gumbel Top-$k$, our extension\footnote{The argmax of a Gumbel softmax would only point to one rule. To select multiple rules we apply a differentiable top-k operator.} of the Gumbel Softmax 
, to choose $k$ rules out of $n$ for any of the conditions in $\mathcal{C}$. 
As discussed in Section \ref{sec:composition}, the ability to choose rules allows \upps to compose modules. 
Typically $k=\frac{n}{2} > 1$ is the best choice of chosen rules as it allows a maximum combination of sub-modules (see Appendix \ref{app:theory} and \ref{app:k_N_ablation} for in-depth theoretical and empirical analysis). 
The rest of the steps are similar to a Transformer layer's. 
However, instead of using all attention heads $\widetilde{\mathrm{H}}_{1 \ldots n}$ to transform a sequence, condition sub-sequences use a varying subset of $k$ heads. 
At step $t$, our prompt vector $\vP_t$ is generated in the following way:
\begin{align}
   \vP_t(\vc_t) &=  \vW^o\sum_{r \in \{r\}} \widetilde{\mathrm{H}}_{r}(\text{LN}(\vc_t)) \label{eq:cond_p},
\end{align}
where LN is layer normalization 
, $\vc_t$ is the condition word embedding at step $t$, $g(\cdot)$ is the transition function (i.e., intermediate Transformer layer), $\vW^o$ is a learnable weight matrix, $\Vert$ is the concatenation operator, $\{r\}$ is the set of indices identifying selected rules 
$\vP_t$ can be applied $L$ times for an $L$ layer \upps.
Note that each \emph{\upps layer shares weights} to allow greater parameter efficiency and knowledge reusability.

\section{Main Results}
\label{sec:main_results}

In the next sections, we discuss our experimental setup and evaluation results across the four settings seen in Fig.~\ref{fig:evaluations} for both full data and low resource settings. 
\citeauthor{t5} observe significant gaps between full-model fine-tuning (FFT) and adapters \cite{pmlr-v97-houlsby19,nmt-adapter} on both summarization and translation tasks. 
The gap increases with increasing training data size. 
Recently, several methods claimed results on-par with FFT on NLU benchmarks \cite{guo2020parameter,hu2021lora} such as GLUE \cite{glue} or on relatively simple CNLG tasks such as E2E \cite{novikova-etal-2017-e2e}. 
However, our results indicate that such PLM adaptation may not generalize well across various architectures and for more complex tasks, high-resource benchmarks or large multi-task datasets. 
As demonstrated in this section (details below), \upps outperforms all other model adaptation techniques on \textbf{three high-resource benchmarks, two datasets in the low resource knowledge transfer setting and one high resource multilingual translation dataset}.

\subsection{General Setup}
\label{sec:gen_setup}

\begin{table}[h]

\begin{center}
\begin{footnotesize}

\begin{tabular}{l|c|ccc|c|c}
\hline
Model & adapted & \multicolumn{3}{c|}{condition}  & shared  & select \\
Name  & layers                       & n/a & emb & txt                     & weights & rule   \\
\hline
\adapter$^1$ & 100\%    & $\surd$ & & & & \\
\camtl$^2$   & 50\%     & & $\surd$ & & & \\
\prefix$^3$  & 100\%    & & & & & \\

\prefixplus* & 100\%    & & & $\surd$ & & \\
\trsfp*      & 33\%     & & & $\surd$ & & \\
\utrsfp*     & 33\%     & & & $\surd$ & $\surd$ &\\
\highlight{\textbf{\upps}} & 33\%     & & & $\surd$ & $\surd$ & $\surd$ \\

\hline
\end{tabular}
\end{footnotesize}

\caption{\small Baselines. $^1$\protect\cite{adapter}, $^2$\protect\cite{pilault2021conditionally}, $^3$\protect\cite{prefix}. *Our implementation. hid=hidden states; emb=task embedding; txt=textual input.}\label{tab:baselinemodels}
\end{center}

\end{table}

\paragraph{Baselines:} We compare \upps to four main model categories: fully fine-tuned (FFT) models
, adapters where trainable layers are inserted between frozen PLM layers, UPG and CPG described in Section \ref{sec:method}. 
A comparison of baselines, including Transformer-Prompt (\trsfp) and Transformer-Prompt with shared weights (\utrsfp) generators, is found in Table \ref{tab:baselinemodels} where we show the percentage of PLM layers adapted, whether the PLM hidden states or the attention is adapted, the types of conditions used (none, embeddings or text prompts).
The \adapter model uses trainable MLPs inserted in the PLM to transform hidden states. 
\camtl is type of multitask adapter that uses task embedding to conditionally adapt hidden states and the attention of the PLM.
\prefixplus is exactly the same as \prefix however we additionally append conditions textual prompt to the input text (e.g.: ``Abstractive Summarization, BBC, Sports'' or ``Jump Twice Jump Right'').
The baselines allow us to measure the effects of conditions (\prefix vs. \trsfp or \prefixplus), the affects of sharing weights (\trsfp vs. \utrsfp) and the affects of rule selection (\utrsfp vs \upps). 
We also provide side by side comparisons of \prefix, \utrsfp and \upps in Figure~\ref{fig:baselines} of the Appendix.
More details are found in Appendix \ref{app:baselines}.

\paragraph{Datasets, Training and Evaluation:} 
We study four Conditional Natural Language Generation (CNLG) datasets SCAN \cite{scan}, Europarl \cite{europarl}, XSum \cite{xsum} and Topic-CNN-DM \cite{topic-aware}.
The datasets and tasks for our experiment were chosen since text conditioning from metadata and instructions is similar to controllable text generation. 
We describe our datasets, training and evaluation setup in Appendix \ref{app:more_gen_setup}. 
\emph{The code and datasets will be made publicly available.}

\subsection{Can CPG Modularization Better Compose Instructions by Composing the Input Text?}
\label{sec:input_text}

\begin{table}[h]
    \footnotesize
    \centering
    \begin{tabular}[b]{lccccc}
        \toprule
        \multirow{3}{*}{\textbf{Model}} &      & Add  & Jump   
        &        &       \\
                                        & Add  & Turn & Around 
        &        &  Avg. \\
                                        & Jump & Left & Right  
        & Length &  MCD  \\
        \midrule   
        \tfive$^1$                  &  98.3 & 69.2 & 99.9  & 5.2  & 10.1 \\ 
        \prefix                     &  83.8 & 66.7 & 91.0  & 5.9 & 7.8 \\  
        \prefixplus                 &  85.1 & 68.1 & 92.3  & 6.4 & 8.5 \\  
        \utrsfp                     &  90.9 & 70.3 & 93.6  & 9.3 & 9.9 \\
        \highlight{\textbf{\upps}}  &  \textbf{99.2} & \textbf{72.9} & \textbf{100}  & \textbf{11.4} & \textbf{12.6} \\
        \bottomrule
    \end{tabular}
    \caption{\label{tab:semanticparsing-results} \small Compositional Generalization test accuracy from exact match on SCAN. $^1$Results from \protect\citeauthor{Furrer2020CompositionalGI}. 
    }
\end{table}

In the first type of evaluation depicted in Fig.~\ref{fig:evaluations}, we test whether modular CPGs such as \upps can compose instructions from the input.
We fist want to test if the model learns from training inputs such as ``jump right'' and ``walk left'', how to properly resolve unseen combinations ``jump left'' or ``walk right'' at test time.
We use a compositional generalization task made from the synthetic dataset SCAN \cite{scan}.
The SCAN dataset is constructed from a rule-based grammar that explicitly composes primitive commands.
\citeauthor{scan}
proposed various data splits to measure \emph{Compositional Generalization} 
by testing on commands constructed with new primitives, excluded complex command sequences or extended length (length generalization).
All adaptation techniques are built on top of a pretrained T5 model.
In the ``Add Jump'' split for example, the training set includes all of the compositional tasks excluding ``Jump'', which is only seen in isolation, and the testing set includes all compositions using ``Jump''. 
The inductive bias in \upps allows it to modularize knowledge via a sparse selection of attention heads.
Experimental results in Table \ref{tab:semanticparsing-results} show that \upps's internal knowledge structure is also composable. 
Not only does \upps surpass all other adaptation methods across splits by a large margin, \textbf{it also improves over \tfive FFT on five out of seven cases}, including on the more difficult Maximum Compound Divergence (MCD) datasets \cite{keysers2020measuring}. 
We follow the training setup of \citeauthor{Furrer2020CompositionalGI}.

\subsection{Can CPG Modularization Better Transfer and Compose Jointly Trained Tasks?}
\label{sec:multitask}


Now, we turn our attention to the second evaluation case in Fig.~\ref{fig:evaluations}: task composition. 
We hypothesize that modular CPG networks such as \upps will better segregate and compose knowledge from other tasks trained jointly. 
We select the Europarl \cite{europarl} multilingual dataset since it contains numerous parallel language corpora.
The Europarl multilingual corpora is also interesting since our instructions can leverage ``bridge tasks''.
For example, we can translate English to French (en $\rightarrow$ fr) at test time if we only trained on English to German (en $\rightarrow$ de) and German to French (de $\rightarrow$ fr) pairs (i.e., translate en to de and then de to fr to get en to fr). 
Further, it was shown that multilingual translation can benefit from positive transfer from one language (task) to the other.
As previously discussed, it is typically easier to outperform fully-finetuned models with parameter efficient adaptation methods in the low-resource setting.
However, it becomes increasingly difficult as the data size scales \cite{t5}.
We evaluate the performance $\overrightarrow{\text{yy-xx}}$ and $\overleftarrow{\text{yy-xx}}$ where $\text{yy},\text{xx} \in \{\text{de,fr,en,es}\}$ and $\text{yy} \neq \text{xx}$.
With most multi-task learning frameworks, it is paramount to maximize positive transfer while minimizing negative transfers.
Since \upps is able to choose up to $k$ rules to apply on a condition, multiple tasks can choose either common shared rules or independent rules,  we believe this mechanism helps enhance positive transfer (shared rules), and mitigate negative transfer (independent rules). 
We present our \mbart \cite{mbart} based multilingual translation results below. 

\begin{figure}[H]
\centering
\begin{tabular}{c}
\adjustbox{valign=b}{
    \begin{tabular}{@{}c@{}}
    \subfloat[\small Europarl trained with 50 samples\label{subfig:europarl50}]{%
        \includegraphics[width=.90\linewidth]{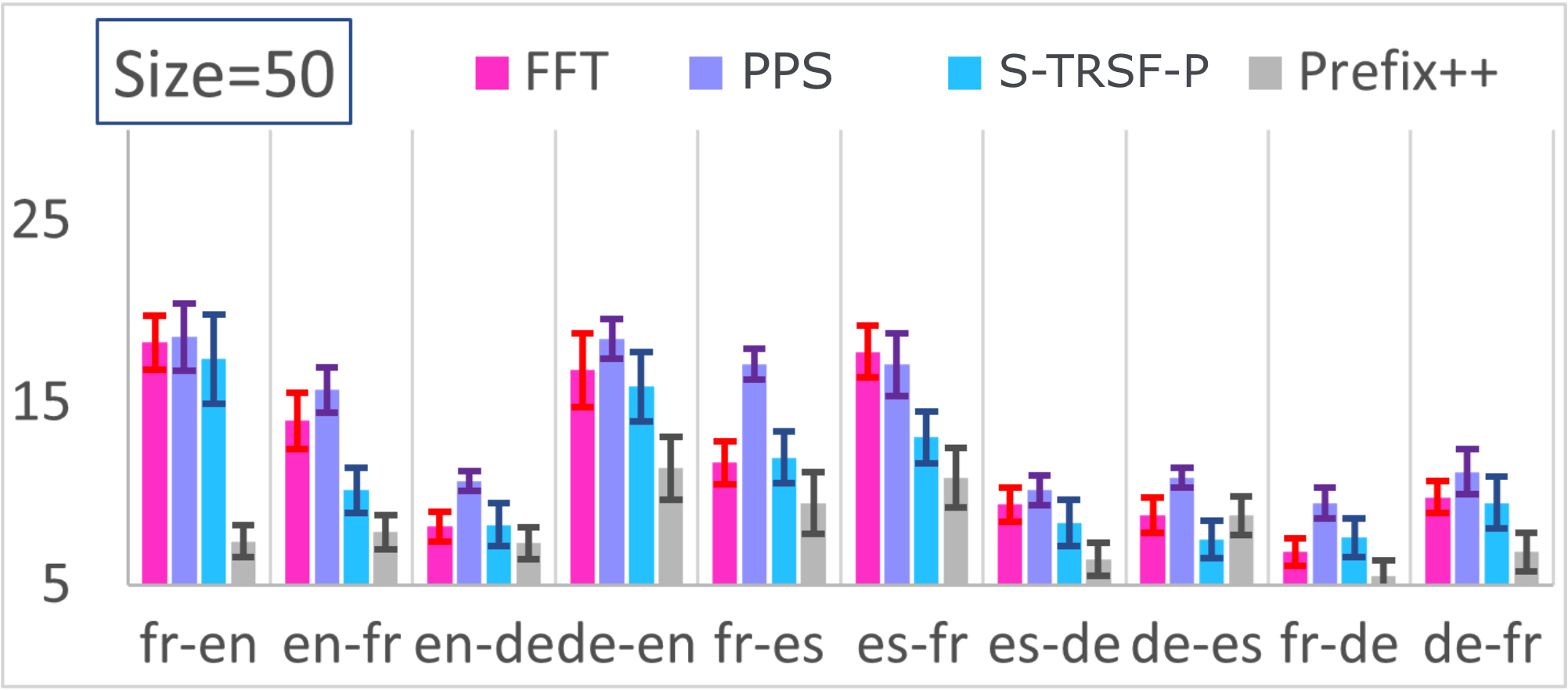}} 
        \vspace{-5pt} \\
    \subfloat[\small Europarl trained with 500 samples\label{subfig:europarl500}]{%
        \includegraphics[width=.90\linewidth]{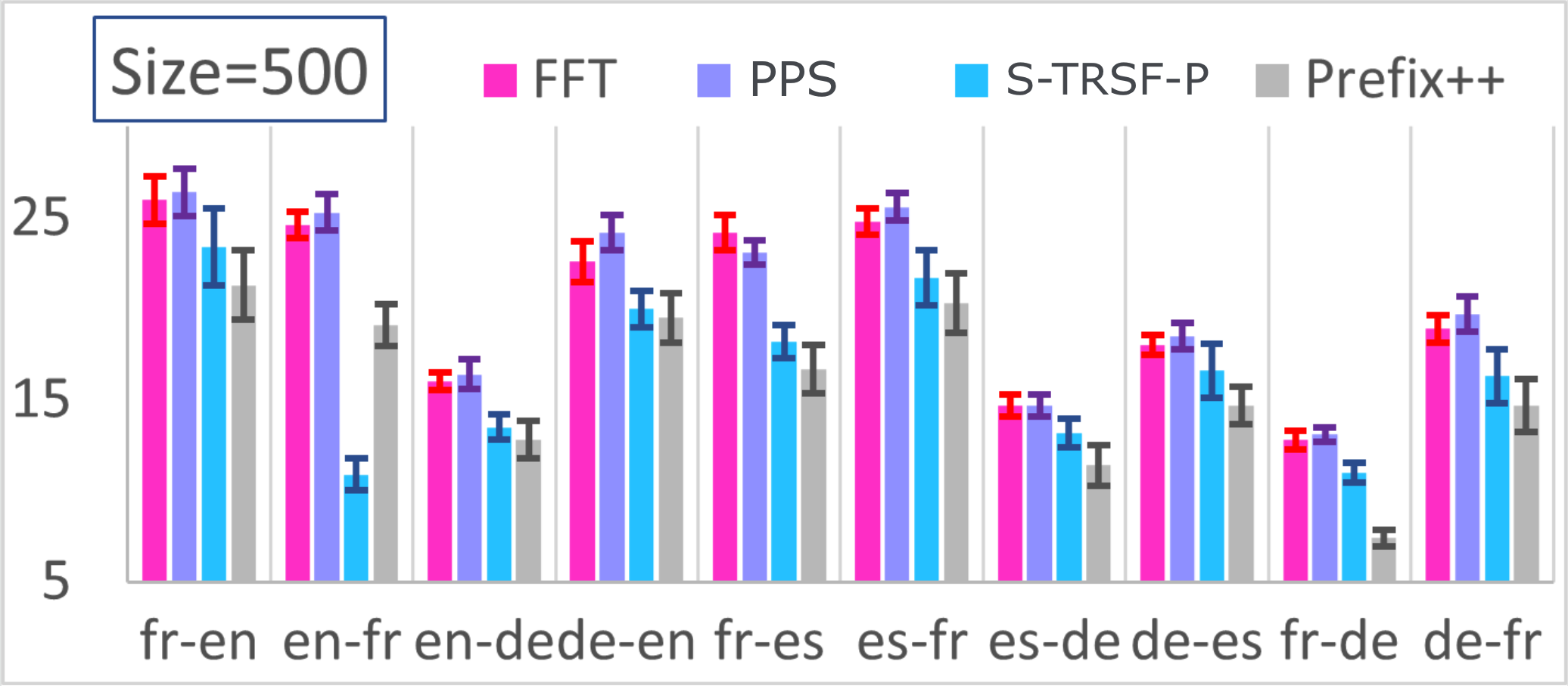}}
    \end{tabular}
}
\end{tabular}
\caption{\small Few-shot multitask. \textbf{PPS = \upps}. BLEU shown for 10 language pairs. Error bars are standard deviations of 3 random seeds.}\label{fig:low-rez-mtl}
\end{figure}

\paragraph{The few-shot case:}
In this experiment, each of the 10 language pairs include $50$ or $500$ training samples.
The task instruction have the form ``Translate \{source\} to \{target\}''. 
For \camtl, each translation pair generates different adapter weights from its task embedding. We do not run experiments for \adapter since \citeauthor{pilault2021conditionally} show that \camtl is better equipped to handle multiple jointly trained tasks.
In Figure~\ref{fig:low-rez-mtl}, we notice that both Condition Prompt Generators (\upps and \utrsfp) perform much better than \prefixplus. 
\upps improves over all other methods on both average performance and standard deviation, including FFT, for \textbf{9 out of 10 tasks} when each pair has $50$ training examples, and for \textbf{8 out of 10 tasks} when each pair has $500$ training examples.

\begin{table}[ht!]
    \footnotesize
    \centering
    \begin{tabular}[b]{l|c|cccc}
        \hline
        \multirow{2}{*}{\textbf{Model}} & \textbf{trained}  & \multicolumn{4}{c}{Europarl}\\
                          & \textbf{params}    & $\overrightarrow{\text{fr-en}}$ & $\overleftarrow{\text{fr-en}}$ & $\overrightarrow{\text{de-en}}$ & $\overleftarrow{\text{de-en}}$ \\
        \hline
        \mbart   & 100\%  & 40.10 & 39.54 & 37.32 & 30.19 \\
        \hline
        \adapter   & 14\% &  27.22 &  26.56 &  24.25 & 20.10 \\
        \camtl    & 13\%  &  36.81 &  36.74 &  33.07 & 26.17 \\
        \prefix      & 15\%  &  27.51 &  26.24 &  26.01 & 20.87 \\
        \prefixplus  & 15\%  &  28.75 &  27.40 &  26.93 & 21.38 \\
        \utrsfp   & \textbf{10\%}  &  36.76 &  36.66 &  33.02 & 26.24  \\
        \highlight{\textbf{\upps}} & \textbf{10\%} & \textbf{38.17}$^\dag$ & \textbf{38.09}$^\dag$ & \textbf{35.20}$^\dag$ & \textbf{28.39}$^\dag$ \\
        \hline
    \end{tabular}
    \caption{\label{tab:translation-results} \small Multilingual translation test BLEU results. Results labelled with $\dag$ are significantly better than all other adapters based on pair-wise significance testing \protect\cite{koehn-2004-statistical} with p = 0.01.
    }
\end{table}

\paragraph{Full data multilingual experiments:}
Each language pair direction has 1M training and 100K testing examples. 
In the full data multitask regime, we see that fully finetuned \mbart has an edge over adaptation methods.
Our results demonstrate that almost all PLM adaptation struggle for a 4M combined dataset size. On average, \upps is 1.8\% lower than FFT but also 1.8\% higher than the second highest performing adaptation method. 
Comparing the results of \upps and \utrsfp to \prefixplus, it is clear that conditional prompt generation performs better than unconidional prompt generation.
This is also true when we compare \adapter to \camtl (Figure~\ref{fig:adapters}), indicating that task-conditioned PLM adaptation is better suited at transferring knowledge.
We start noticing that \upps's modularization of knowledge from learnable rules and the control structure has clear benefits over \utrsfp.

\subsection{Does Conditioning Prompt Generators on Data Metadata Improve Performance?}
\label{sec:meta-info}

\begin{table*}[ht!]
    \footnotesize
    \centering
    \setlength\tabcolsep{4.0pt}
    \begin{tabular}[b]{l|c|c|cccc|cccc|cccc}
        \hline
        \multirow{2}{*}{} & \textbf{trained}  & \textbf{prompt} & \multicolumn{3}{c|}{XSum} & \multicolumn{3}{c|}{XSum-OOT} & \multicolumn{4}{c}{Topic-CNN-DM} \\
                          & \textbf{params}   & \textbf{length} & \textbf{R1} & \textbf{R2} & \textbf{RL} & \textbf{MET}  
                                                                & \textbf{R1} & \textbf{R2} & \textbf{RL} & \textbf{MET}       
                                                                & \textbf{R1} & \textbf{R2} & \textbf{RL} & \textbf{MET}      \\
        \rowcolor{gray!20}\multicolumn{13}{c}{\it Full model fine-tuning}                          \\
            \bart$^1$ & 100\% & ---  & 45.14 & 22.27 & 37.25 & 34.03                    
                                    & 39.53  & 16.88 & 31.55 & 28.42                                   
                                    & 39.38  & 22.41 & 35.10 & 29.93        \\
        \rowcolor{gray!20}\multicolumn{13}{c}{\it Adaptation without conditions}                        \\
            \adapter \ding{163} & 14\% & --- & 43.29 & 20.75 & 34.91  & 32.03                                      
                                            & 38.51 & 15.76 & 30.28  & 26.93                                      
                                            & 37.19 & 20.90 & 34.01 & 27.89                                      \\
            \prefix$^2$  & 15\% & 200 & 43.80 & 20.93 & 36.05 & 32.66                                         
                                     & 39.41 & 16.87 & 31.47 & 28.09                                                  
                                     & 38.86 & 21.44 & 34.81 & 28.41                                                  \\
        \rowcolor{gray!20}\multicolumn{13}{c}{\it Adaptation with conditions}                          \\
        
            \prefixplus  & 15\% & 200 & 43.90 & 20.98 & 36.14  & 32.74                                         
                                           & 39.52 & 16.87 & 31.49 & 28.17                                        
                                           & 38.87 & 21.52 & 35.02 & 28.59                                        \\
                                           
            \trsfp                     &  36\% & 150 & \myul[black]{44.39} & \myul[black]{21.41} & \myul[black]{36.26} & \myul[black]{33.29}                                    
                                       & \myul[black]{39.82} & \myul[black]{17.06} & \myul[black]{31.71} & \myul[black]{28.96}                                
                                       & \myul[black]{39.32} & \myul[black]{23.06} & \myul[black]{36.25} & \textbf{32.07} \\
            
            \utrsfp & \textbf{10\%}   & 150   & 44.25 & 21.36 & 36.10 & 33.09 
                                      & 39.53 & 16.95 & 31.49 & 28.28                                             
                                      & 38.93 & 22.02 & 35.38 & 29.77                                             \\
            
            \highlight{\textbf{\upps}}  & \textbf{10\%} & 150 & \textbf{44.67}$^\dag$ & \textbf{21.64}$^\dag$   & \textbf{36.52}$^\dag$  & \textbf{33.64}$^\dag$          
                                        & \textbf{40.03}$^\dag$ & \textbf{17.25}$^\dag$ & \textbf{31.90}$^\dag$ & \textbf{29.07} 
                                        & \textbf{39.94}$^\dag$ & \textbf{23.77}$^\dag$ & \textbf{36.90}$^\dag$ & \myul[black]{32.05}        \\
        \hline
    \end{tabular}
    \caption{ \small
    Summarization results. \ding{163}: Applying method of~\protect\citeauthor{adapter} on \bart. \textbf{best} and \myul[black]{2nd best} results indicated. 
    $^1$XSum results from \protect\citeauthor{bart} and Topic-CNN-DM results from \protect\citeauthor{topic-cnndm}. $^2$XSum and XSum-OOT results from \protect\citeauthor{prefix}.
    Results labelled with $\dag$ are significantly better than all other adapters with p = 0.05.
    MET=Meteor.
    }\label{tab:summarization-results}
\end{table*}

In this section, we assess \upps and our baseline's performance when text conditioning is the news article metadata (third evaluation case).
Our results are in Table~\ref{tab:summarization-results}.
The instructions remain unchanged (i.e., ``Abstractive Summarization'') for each dataset and example. 
However, the metadata such as news outlet, article type or news topic (for Topic-CNN-DM) change for every example.
XSum-OOT is a new dataset based on XSum where the training set consists of specific types of articles (i.e., ``world'', ``politics'', ``business'', ``sports'') and the test set is composed of
the remaining news categories (i.e, ``health'', ``tech'').
XSum-OOT tests if models can generalize to unseen news categories and extrapolate from seen data metadata.

Overall, we notice that \upps, other CPGs and even \prefixplus all perform better with the added information. 
We then observe that \upps consistently surpasses \adapter, \prefixplus and our Conditional Prompt Generation baselines (\trsfp and \utrsfp). 
We note that \upps and \utrsfp adapt 33\% of the layers with 5\% fewer trained parameters, 50 fewer prompt vectors when compared to \prefix.
\textbf{\upps is the only adaptation technique that is consistently on par or outperforms FFT}.
For close to 4 times less parameters, \upps consistently outperforms \trsfp with statistical significance (except on Meteor scores), surprisingly offering a 1.5\% relative improvement over \trsfp  on averaged scores across the three datasets. 
When the task requires more control on the generated output text, we see that \upps rule selection mechanism boosts relative performance by close to 4\% on Topic-CNN-DM \cite{topic-aware} over \utrsfp. 
The improvement in the controlled generation task shows that the separate transformation of conditions help disentangle representations of the topic and other accompanying conditions.

\subsection{Can we Compose Instructions, Metadata and Pre-Learned Support Task Knowledge?}
\label{sec:composition}

In this section, we test the affects of pre-learning specific Conditional Language Generation support tasks along with their instructions. 
At test time, we evaluate our methods on Abstractive Summarization using an instruction that was textually composed from the pre-learned task instructions.
We handpick three support tasks that are related to summarization.
Indeed, summarization can be decomposed into several arguably simpler tasks \cite{pilault-etal-2020-extractive}: 
(1) \textbf{NLU -} understand and interpret a source document, 
(2) \textbf{Extract -} highlight and prioritize important points,
(3) \textbf{Paraphrase -} coherently reformulate key messages.
We describe a technique to induce transfer and composition via text conditioning. 
We vary the \emph{overlap of words} describing a support task and a target task, and measure the effects on target task zero-shot and few-shot performance. 
In this setting, support tasks are trained jointly and the target task is evaluated separately. 
We provide examples of how to create target task descriptions from support task descriptions with certain word overlaps, which is color coded for clarity in Table \ref{tab:task_description}.
Note that we did not search over multiple instructions in the experiments below.
The experimental setup and the support tasks datasets used or created are discussed in Appendix~\ref{app:more_gen_setup}. 

\begin{figure}[ht!]
    \begin{center}
        \includegraphics[width=0.49\textwidth]{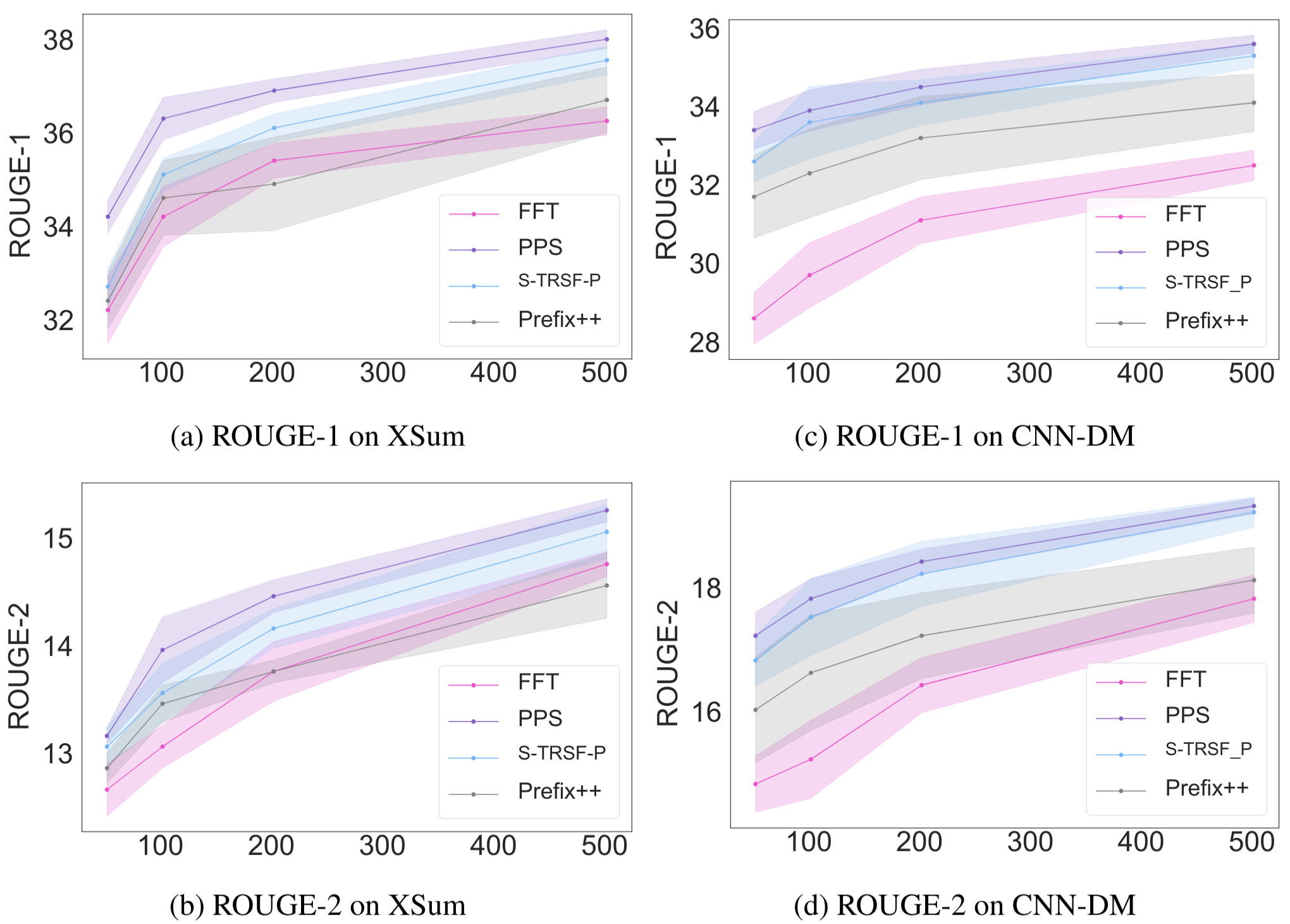}
    \end{center}
    \caption{\small \label{fig:low-rez-transfer} Few-shot transfer. \textbf{PPS = \upps}. Axis show average ROUGE scores vs. number of training examples. Error bands represent min/max of three runs with random seeds.}
\end{figure}

\paragraph{The few-shot case:}
For this experiment, we use the same set-up as in the zero-shot section, except our target task is fine-tuned on a few examples.
The intermediate task prompts used in this section are presented in Table \ref{tab:task_description} in the ``simple'' task instruction column.
We also ran experiments on Topic-CNN-DM to see if our results generalized to other datasets.
Our few-shot summarization results are shown in Figure~\ref{fig:low-rez-transfer}. 
We see that \upps average performance is above that of other \bart based techniques while also maintaining a lower spread between maximum and minimum values. 
This means that \upps better leverages prior task instructions and knowledge to initialize new task prompts. 
Again, for controlable generation (Topic-CNN-DM), \upps displays close to a \textbf{5 and 2 ROUGE-1 point and 3 and 2 ROUGE-2 point improvement} over FFT and \prefixplus respectively.

\paragraph{The zero-shot case:}
We first test if pretrained task and \emph{``Simple task instructions''} in Table~\ref{tab:task_description} are enough to solve our target task without any additional fine-tuning.
CPG modularization approaches such as \upps should improve the composition of task instructions by reusing words and phrases found in support tasks and further enhance \upps's rule reusability and knowledge composition ability.
From our results in the \emph{``Simple''} column of Table \ref{tab:zero-shot}, we first notice that \upps tops \utrsfp, our strongest baseline, by $0.8$ R1 or a 4\% relative improvement.

\begin{table}[ht!]
\footnotesize
\begin{center}
\begin{tabular}{l|c|c|c}
\toprule
\multirow{2}{*}{Method}       & \multicolumn{2}{c|}{Task Instructions} & \multirow{2}{*}{$\Delta$ R} \\
                              & \textbf{Simple} & \textbf{Detailed} &  \\
\midrule
\rowcolor{gray!20} & \it \textbf{R1 / R2} &   \it \textbf{R1 / R2} & \it \textbf{R1 / R2} \\
\bart \tiny{FFT}    & \g{20.9}{1.1} / \textbf{\g{4.0}{0.2}} & \g{21.4}{0.7} / \g{4.2}{0.1} &  \textbf{+0.5} / +0.3  \\
\prefix      & \g{19.0}{0.5} / \g{3.5}{0.1} & \g{19.1}{0.5} / \g{3.5}{0.1} & +0.1 / +0.0  \\
\utrsfp      & \g{21.0}{0.8} / \g{3.9}{0.1} & \g{21.2}{0.5} / \g{4.1}{0.1} & +0.3 / +0.1 \\
\upps        & \textbf{\g{21.8}{0.6}} / \textbf{\g{4.0}{0.1}} & \textbf{\g{22.3}{0.5}} / \textbf{\g{4.4}{0.1}} &  \textbf{+0.5} / \textbf{+0.4} \\

\bottomrule
\end{tabular}
\end{center}
\caption{\small Zero-shot test ROUGE 1 and ROUGE 2 scores on XSum when pretrained with ``simple'' or ``detailed'' task instructions. +/- numbers in grey are the standard deviation (STD).}\label{tab:zero-shot}
\end{table}

\paragraph{Improving task instructions:}
We evaluate if it helps to improve the task instructions content and support-target word overlap.
The detailed target instruction contains word segments from all three support tasks whereas the simple target instruction is composed from only two tasks.
As shown in Table \ref{tab:task_description}, \emph{``Detailed task instructions''} allow us to fully connect each support task with the target task, while providing additional and potentially relevant information. 
Further, the detailed target instruction are lengthier and allow greater overlap between words describing the support tasks as well.

Our \emph{hypothesis} is that greater overlap and lengthier instruction will improve target task performance by better inducing prior knowledge composition.
Task description composition by assembling task descriptions works in conjunction with module composition (see explanation in Appendix \ref{app:mech_composition} and an empirical analysis in Appendix \ref{app:rule_compose}).
Our results suggest that all models benefit from the improved task instruction as observed in the \emph{``Detailed''} column of Table~\ref{tab:zero-shot}.
Specifically, \upps gets the largest \% $\Delta$ gain when moving from ``Simple'' to ``Detailed'' instruction, with a 10\% increase in R2 while maintaining the same standard deviation.

\section{Conclusion}



In this work, we evaluated conditional and composable architectures to conditionally generate differentiable continuous prompts. 
Our new modular prompt generator \upps with sparse rule selection outperforms strong Transformer-based prompt generator baselines. 
\upps excels in the ability to distill knowledge from the conditioned task and example metadata, and often surpasses the fully fine-tuned model, with fewer trainable parameters than other prompt generators. 
In general, \upps performed well on three tasks -- summarization, translation, and semantic parsing, with three different models (\bart, \mbart, \tfive). 
\upps provided even larger improvements in controlled CNLG, where disentangling topic information from other conditions is important.
The improvement is even larger for task which requires more controlled generation. 
\upps showed ability to compose knowledge from bridge tasks to enhance transfer learning and improve unseen tasks in few-shot setting both from a mechanistic and induced standpoint.
Finally, via both theoretical (see Appendix~\ref{sec:analysis}) and empirical analysis, we showed  how rules are sparsely selected and composed, the key ingredients that make \upps more successful and generalizable.

\clearpage
\bibliographystyle{named}
{\small
\bibliography{ijcai23}}

\clearpage
\appendix

\section*{Clarifications and Recap}

Before going through the appendix material, we will first do a small recap of \upps.
As mentioned earlier in the article, \upps is a mechanism that learns to transform task instructions or input metadata into continuous prompts by conditionally composing representations from a differentiable set of sparsely selected modules (rules).
The model is useful for transferring the learning and can reconstruct prior prompt knowledge using rules that helps surpass baselines in several categories and uses less trained parameters.

\paragraph{What are differentiable rules?} 
We use the definition from \cite{goyal2021neuralprodsys} that rules, ``carried
over from the traditional symbolic system, operate on inputs that are bound by conditions. In the deep learning implementation, a production rule is represented by a distinct MLP'' but we use instead an attention head (see Section~\ref{sec:app_conditions}). 
$[\mathbf{R}]$ embeds the position of $N$ rules and $[\mathbf{M}]=[\mathbf{E}]\cdot[\mathbf{R}]^T$ maps a condition to a rule. 
With top-$k$ on $[\mathbf{M}]$, different conditions activate different rules, e.g.: for $k=2$ and 8 different rules to choose from for example, \emph{IF} ``summarization'', \emph{THEN} $\mathbf{I}_1$ from rules 1 and 2; \emph{IF} ``bbc'', \emph{THEN} $\mathbf{I}_2$ from rules 4 and 7.

\paragraph{How are rules chosen?}
The embedding matrices in Figure~\ref{subfig:pps_layer} are the Modus Ponens inference rules:
\begin{enumerate}
\item conditions are embedded giving $[\mathbf{E}]$; 
\item multiplied by a Rule matrix $[\mathbf{R}]$ which gives the Rule-Condition matrix $[\mathbf{M}]$; 
\item a Gumbel top-$k$ on $[\mathbf{M}]$ selects $k$ rules (attention heads); 
\item the selected $k$ rules transform the input to give instructions $\mathbf{I}$; 
\item instructions representation $\mathbf{I}$ is used as a differentiable prompt in the LLM.
\end{enumerate}

Our \trsfp baseline selects all rules since it uses all heads. 
We find that sparsely choosing rules with \upps gives better results.
$H$ represents all attn heads for the \trsfp conditional prompt generator while $\hat{H}_i$ represents one of selected attn heads in \upps.

Our code is publicly available in the supplementary material which will be also be made available on GitHub.

\section*{Appendix}

The appendix contains important information about \upps and its baselines. 
Sections \ref{sec:app_conditions}, \ref{sec:connect_pps_ps} and \ref{app:mech_composition} are a direct extension of the Methodology Section \ref{sec:method} and provide additional details on conditions, the link between \upps and ProdSys and module composition. 
Further, due to page limitations, we have moved all of our analysis of \upps to Section \ref{sec:analysis} which provides valuable insights into: 
(1) \upps theoretical generalization bounds  (Section \ref{app:theory}); 
(2) ablations on the number of top $k$ rules chosen amongst $N$ possible options (Section \ref{app:k_N_ablation});
(3) module reusability statistics when introducing a new task (Section \ref{app:rule_compose}); 
(4) results assessing the affects of bridge tasks in task composition (Section \ref{app:bridge_compose}); 
(5) findings that show that \upps can sparsely associate a single rule to a single task.
Moreover, we provide in Sections \ref{app:more_gen_setup}, \ref{sec:datasets} and \ref{app:baselines} more information on our general experimental set-up, our datasets and our baselines. 
Finally, in Section \ref{sec:more_theory}, we present our assumptions and proofs. 

\section{Additional details on conditions and rules}
\label{sec:app_conditions}

\begin{algorithm}[th]

    \begin{algorithmic}[1]
    \footnotesize
    
    \item[]
    
    \State{\bfseries \itshape Input:} Conditioning text word embeddings $\mathcal{S}_C = \langle \vc_t | t \in \{1, \ldots,  T_C\}\rangle_C$ for each condition $C \in \mathcal{C}$, rule embedding $\Vec{\vR_i}$, $n$ \upps attention heads $\widetilde{\mathrm{H}}_{1 \ldots n}$ corresponding to each rule $\vR_{1 \ldots n}$ and top $k$.
    
    \For{each $C \in \mathcal{C}$}
        \State{{\bfseries \itshape Step 1:} \itshape Get the fixed-sized condition representation using Condition Encoder $f(\cdot)$}: $\Vec{\vC} \gets f(\mathcal{S}_C)$
        \vspace{1mm}
        
        \State{{\bfseries \itshape Step 2:} \itshape Select \{top-k rule, condition\} pairs, for $\epsilon \sim \mathrm{Gumbel} (0,1) \text{, } \overline{\vW}^q \text{ learnable weight}$}
        \State \bull $\begin{bmatrix} \vR \end{bmatrix} \gets (\Vec{\vR}_{1} \ldots \Vec{\vR}_{n})$ \text{ and } $\vq \gets \Vec{\vC} \overline{\vW}^q$
        \State \bull $\{r\} = \{r_1, \ldots, r_k\} = \text{top-k} \left(q\begin{bmatrix} \vR \end{bmatrix} + \epsilon \right) \quad \qquad \qquad \qquad
        \refstepcounter{equation}(\theequation)\label{eq:rules}$
        \vspace{1mm}
        
        \State{{\bfseries \itshape Step 3:} \itshape Select condition context, for $\epsilon \sim \mathrm{Gumbel} (0,1)  \text{, } \widetilde{\vW}^{k} \text{ and } \widetilde{\vW}^{q} \text{ learnable weights}$}
        \State \bull $\begin{bmatrix} \vQ \end{bmatrix} = (\Vec{\vR}_{r_{1}}, \ldots, \Vec{\vR}_{r_{k}})\widetilde{\vW}^q $ \text{ and } $k =\Vec{\vC} \widetilde{\vW}^k$
        \State \bull $ j = \text{argmax} \left(\begin{bmatrix} \vQ \end{bmatrix} k + \epsilon\right), j\in\{1,\ldots,|\mathcal{C}|\}\ \qquad \qquad \qquad
        \refstepcounter{equation}(\theequation)\label{eq:context}$
        \State \bull Let $C'$ be the $j$th condition $\in \mathcal{C}$ and $\mathcal{S}_{C'} = \langle {\vc'}_t | t \in \{1, \ldots,  T_{C'}\}\rangle_{C'}$
        \State{{\bfseries \itshape Step 4:} \itshape Apply selected rules to conditioning text $C$ and its context $C'$, $\text{ for learnable weight } \vW^o$}
        \State \bull $\text{Prompt} = \vW^o\sum_{r \in \{r\}} \widetilde{\mathrm{H}}_{r}(\text{LN}(\text{Concat}(\mathcal{S}_{C}, \mathcal{S}_{C'})))$%

    \EndFor
\end{algorithmic}

\caption{Prompt Production System (\upps)}
    \label{alg:upps}
\end{algorithm}


Table \ref{tab:conditions} is an overview of various conditioning text for each condition used across experiments. 
Information may differ from one task to the other depending on data availability and relevancy to the task. 
Note that max $T_C$ represents the maximum allowable number of tokens after subword tokenization. 
Sequences longer than $T_C$ are clipped and sequences smaller than $T_C$ are padded. 
The inferred news labels are noisy since the labeling model was trained on a different news dataset.
Below is an example of the summarization and semantic parsing conditions covered in Section \ref{sec:main_results}:

\begin{table}[h]

\setlength\tabcolsep{4pt}
\begin{center}
\begin{footnotesize}

\begin{tabular}{llcccrr}
\toprule
Conditions     & CPG Input & \multicolumn{3}{c}{Used in task} & max   & \multirow{2}{*}{from} \\
$\in \mathcal{C}$  &  example      & \textbf{sum} & \textbf{nmt} & \textbf{sp} & $T_C$ & \\
\midrule
\rowcolor{cyan!15}\multicolumn{7}{c}{\it \textbf{Task-specific}} \\
Instructions & ``Translate en to fr'' & $\surd$ & $\surd$ & $\surd$   & 50 & H \\
\rowcolor{cyan!15}\multicolumn{7}{c}{\it \textbf{Example type}}  \\
News outlet & ``cnn'', ``bbc''          & $\surd$ & $\times$ & $\times$ & 5 & D \\
Direction   & ``opposite''          & $\times$ & $\times$ & $\surd$ & 5 & D \\
Conjunction & ``and'', ``after''        & $\times$ & $\times$ & $\surd$ & 5 & D \\
\rowcolor{cyan!15}\multicolumn{7}{c}{\it \textbf{Example-specific}} \\
Input text  & input text            & $\times$ & $\surd$  & $\surd$  & 514 & D \\
News label  & ``sports'', ``politics''  & $\surd$ & $\times$ & $\times$ & 5 & M \\
\bottomrule
\end{tabular}
\end{footnotesize}
\end{center}
\caption{\small Conditioning text inputted to the prompt generator. \textbf{sum} = summarization, \textbf{nmt} = translation and \textbf{sp} = semantic parsing.}
\label{tab:conditions}

\end{table}

This section provides additional details to complement Section \ref{sec:cpg}. We have written an example below on how text conditioning was formulated for the semantic parsing task. Please note that, for semantic parsing, no external information was used. All conditions come from the SCAN dataset.

\begin{example} \label{ex:ex_scan} Consider semantic parsing, where a sequence of actions is generated given an input command. 
For instance, the command ``run thrice and walk opposite left'' must produce actions ``I\_RUN I\_RUN I\_RUN I\_TURN\_RIGHT I\_WALK''. 
A human provides instructions on how to perform the task: ``generate actions'' (condition 1). 
We notice that example-specific words such as ``opposite'' (condition 2) or conjunctions such as ``and'' or ``after'' (condition 3) are important to the choice and order of generated actions.
\end{example}
The following is an example for news summarization
\begin{example} \label{ex:ex_summ} Consider news summarization, where an article is compressed into a shorter text. 
A human provides instructions to perform the task: ``summarize abstractively'' (condition 1). 
We notice that target summaries are slightly different if the news outlet is ``CNN'' or ``Daily Mail'' (condition 2) or if the article covers ``sports'', ``business'', ``politics'' (condition 3). Conditions 2 and 3 may improve summaries by providing useful insights on example types.
\end{example}

\begin{table}[ht!]
\setlength\tabcolsep{3.5pt}
\begin{center}
\begin{footnotesize}
\begin{tabular}{l|l|l}
\hline
                  & \textbf{Simple task instructions} & \textbf{Detailed task instructions}   \\
\rowcolor{cyan!15}&\multicolumn{2}{c}{\it \textbf{Support Tasks}} \\

$^1$     & Entity Answering   & \textcolor{orange}{Find the most common entities}          \\
$^2$    & Extractive \textcolor{orange}{Summarization}   & \textcolor{purple}{Find the most relevant sentence}    \\
$^3$     & \textcolor{blue}{Abstractive} Parapharase  & \textcolor{blue}{Abstractive rephrase sentences}   \\
\rowcolor{cyan!15}&\multicolumn{2}{c}{\it \textbf{Target Task}}      \\
$^4$     & \textcolor{blue}{Abstractive}    & \textcolor{purple}{Fi}\textcolor{orange}{nd} the \textcolor{purple}{most relevant sentences}           \\
 & \textcolor{orange}{Summarization}  &  and \textcolor{orange}{most common entities}  then        \\
            &                  &  \textcolor{blue}{abstractively rewrite sentences}        \\
\hline
\end{tabular}
\end{footnotesize}
\caption{ \small Composing target tasks instructions from support tasks (simple or detailed) for XSum. Colors show word overlap.}\label{tab:task_description}
\end{center}
\end{table}

The news labels are extracted by applying a news category classifier. There are seven categories including entertainment, sports, business, politics, health, science and other. 
The classifier is based on a pretrained BERT 
model.
The hidden states are passed through an LSTM with attention. 
The pretrained model, code and news classification dataset that we have used can be found at \url{https://github.com/rootally/News-Category-Classification-with-BERT.git}.

\upps use differentiable rules similar to NPS. The same definition of differentiable rules applies to us. 
\citeauthor{goyal2021neuralprodsys} defines it as: 
“We describe an end-to-end deep learning model that constructs representations of entities, and then operates on these entities with differentiable—and thus learnable—production rules. 
The essence of these rules, carried over from the traditional symbolic system, is that they operate on variables that are bound, or linked, to entities. 
In the deep learning implementation, each production rule is represented by a distinct MLP with query-key attention mechanisms to specify the rule-entity binding and to determine when the rule should be triggered for a given entity. 
We are not the first to propose a neural instantiation of a production system architecture. 
\citeauthor{TOURETZKY1988423} gave a proof of principle that neural net hardware could be hardwired to implement a production system for symbolic reasoning.”
For NPS, a differentiable rule is an MLP. 
For \upps, a differentiable rule is one of multiheads in the attention mechanism.

\section{More details connecting \upps to ProdSys}
\label{sec:connect_pps_ps}

We can now more formally make the connection between the \upps prompt generator and ProdSys that we sketched out in Figure~\ref{fig:method}. 
Algorithm \ref{alg:upps} shows that prompts are generated from a control mechanism that maps a condition $C_j$ to a rule $R_i$ in the form $C_j \xrightarrow[]{} R_i$. 
The rules are sparsely selected depending on the condition (see Equation \ref{eq:rules}). 
From Equation \ref{eq:cond_p}, we notice that generated prompts are a function of condition embeddings which contains task/example knowledge (working memory). 
Finally, we observe from Equations \ref{eq:prompt-adapter} and \ref{eq:gate} that the prompt-based adaptation to PLM head $H$ is both a function of $\vP(\vc_t)$ and $\vQ, \vK, \vV$, which contain pretraining knowledge (long-term memory) in the frozen weights.

Since we use differentiable rules it is not necessary to predetermine human-engineered rules as in ProdSys.
However, this strength also makes \upps less interpretable than ProdSys.
Nonetheless, explainability is an issue with most deep neural networks. 
We however believe that \upps is more interpretable than typical monolithic deep architectures. 
For example, in Figure~\ref{fig:rule_usage}, our model can associate trained tasks with untrained tasks via rule selection. 
Further, since our prompt generator uses attention, we can associate specific words in the condition to specific rules.

\section{Module Composition}
\label{app:mech_composition}

Several studies 
have shown the powerful transfer and generalization abilities of modular neural networks over their monolithic counterparts. 
Typically, model modularization involves the selection of distinct subsets (modules) of a network and a mechanism for sparsely combining the subnetworks. 
In contrast to monolithic structures, only a portion of layers are dynamically activated during training or inference, and each module $r \in \{r\}$ is meant to perform a distinct function $g$ parametrized by weights $\theta_r$. 
In \upps, each set of rules represents separate modules. We rewrite Equation \ref{eq:cond_p} in a simpler form:
\begin{align}
    \vP_t(\vc_t) &= \sum_{r \in \{r\}} w_r \cdot g(\vc_t, \theta_r), \label{eq:compose} \\
    \text{where } g(\vc_t,\theta_r) &= \widetilde{\mathrm{H}}_{r}(\text{LN}(\vc_t))
\end{align}
and where $\theta_r$ are the parameters of $\widetilde{\mathrm{H}}_{r}$ and $w_r=\vW^o$ weights the contribution of each module. 
Apart from the sparse module selection (i.e., $r \in \{r\}$), Equation \ref{eq:compose}, describing \upps module composition, looks very similar to the equations describing other module composition mechanisms \cite{ostapenko2021continual}.

\section{Analysis}
\label{sec:analysis}

\upps is dependent on a few variables: the number of rules in the architecture, the top-k sparsely selected rules, and the informativeness of task instructions given. In this section, we attempt to explain the impact of such variables.

\subsection{Theory: How do $k$ and $N$ affect generalization?}
\label{app:theory}
Consider $\mathcal{T}$ tasks learnt jointly over an input space $\mathcal{X}$ using a function $p: \mathcal{X}\rightarrow \mathbb{R}$. 
For a multi-task representation learning problem, the predictor function can be factorized such that $p=\{f^{\tau} \circ h\}^\mathcal{T}_{\tau=1}$, the composition of a representation function $h\in \mathcal{H}$ used across tasks and the set of task-specific prediction functions $f\in \mathcal{F}$. The representation hypothesis class is defined as $\mathcal{H} = \{h: \mathcal{X}\rightarrow \mathbb{R}^d\}$, where $d$ is the dimension of the representation space. Similarly, the task prediction hypothesis class is defined as $\mathcal{F} = \{\{f^{\tau}: \mathbb{R}^d\rightarrow \mathbb{R}\}^{\mathcal{T}}_{\tau=1}\}$. We consider the multi-task representation learning problem where we minimize the task-averaged empirical risk:

\begin{equation}
\label{eq:mtl}
    \underset{h\in \mathcal{H},f\in \mathcal{F} }{\min}\frac{1}{\mathcal{T}}\sum\limits_{\tau=1}^{\mathcal{T}} l^{\tau}(f^{\tau} \circ h).
\end{equation}

To simplify the problem, we take the loss function for task $\tau$ to map $l^{\tau}: \mathbb{R} \times \mathbb{R} \rightarrow [0,1]$. $l^{\tau}$ is assumed $1$-Lipschitz. If input and output samples are drawn from a data distribution $\{x^{\tau},y^{\tau}\} \sim \mathcal{D}_{\tau}$, the expected risk for task ${\tau}$ is then written as $\mathcal{L}_{\mathcal{D}_{\tau}}(f^{\tau} \circ h)=\mathbb{E}_{\{x^{\tau},y^{\tau}\} \sim \mathcal{D}_{\tau}} \big[ l^{\tau}(f^{\tau}(h(x^{\tau})),y^{\tau}) \big]$. The task-averaged expected risk is then: $\mathcal{L}_{\mathcal{D}}=\frac{1}{\mathcal{T}}\sum_{{\tau}=1}^{\mathcal{T}} \mathcal{L}_{\mathcal{D}_{\tau}}$. For task ${\tau}$ and $n$ training samples from $S_{\tau} = (\bar{X}^{\tau}, \bar{Y}^{\tau}) = \{(x^{\tau},y^{\tau})_i | i\in \{1,\dotsc,n\}\}$, we define the corresponding empirical risk as $\mathcal{L}_{S_{\tau}}(f^{\tau} \circ h)=\frac{1}{n}\sum_{i=1}^n l^{\tau}(f^{\tau}(h(x^{\tau}_i)),y^{\tau}_i) $. 
From Equation \ref{eq:mtl}, the task-averaged empirical risk becomes $\mathcal{L}_{S}=\frac{1}{\mathcal{T}}\sum_{{\tau}=1}^{\mathcal{T}} \mathcal{L}_{S_{\tau}}$.

\textbf{Theorem \ref{thm:gen_err}} 
\cite{JMLR:v6:ando05a}
provides a task-averaged empirical generalization bound for multi-task learning problems.

\begin{theorem}
\label{thm:gen_err} Let $\delta \in (0,1)$. Assuming $f$ is $\rho$-Lipschitz, with probability at least $1-\delta$, for every $h\in \mathcal{H}$ and every  $f \in \mathcal{F}$, the multi-task learning generalization error $\Delta(\mathcal{L})$ is bounded by $\mathcal{B}_{\mathcal{H}}$:

\vspace{-1em}
\begin{multline}
\Delta(\mathcal{L}) = \mathcal{L}_{\mathcal{D}} - \mathcal{L}_{S}  \leq \mathcal{B}_{\mathcal{H}} = \\
c_{1}\frac{\rho G(\mathcal{H}(\bar{X}))}{n\mathcal{T}}+c_{2}\frac{%
Q(\mathcal{F})\sup_{h\in \mathcal{H}}\left\Vert h\left( \bar{X}\right) \right\Vert}{n\sqrt{\mathcal{T}}} + c_{3}, \label{eq:theorem}
\end{multline}%

where $c_{1}$, $c_{2}$ are universal constants\label{thm:unifMTL}, $c_{3} = \sqrt{\frac{9\ln \left( 2/\delta \right) }{2n\mathcal{T}}}$. We denote by $\gamma $ a generic vector of independent standard normal
variables and define the Gaussian average $G(\mathcal{H}(\bar{X}))$ of a subset $\mathcal{H}(\bar{X}) \in \mathbb{R}^d$ as:
\begin{equation}
G(\mathcal{H}(\bar{X})) =\mathbb{E}\sup_{h \in \mathcal{H}}\left\langle \gamma ,y\right\rangle
=\mathbb{E}\sup_{h \in \mathcal{H}}\sum_{i,{\tau}}\gamma _{i}g_{i}(x_{\tau}).
\end{equation}

and

\vspace{-0.5em}
\begin{equation}
Q\left( \mathcal{F} \right) =\sup_{\mathbf{y},\mathbf{y}^{\prime }\in Y,~%
\mathbf{y}\neq \mathbf{y}^{\prime }}\mathbb{E}\sup_{f\in \mathcal{F} }\frac{%
\left\langle \mathbf{\gamma },f\left( \mathbf{y}\right) -f\left( \mathbf{y}%
^{\prime }\right) \right\rangle }{\left\Vert \mathbf{y}-\mathbf{y}^{\prime
}\right\Vert }.
\end{equation}
\end{theorem}

\begin{remark}
\label{rem:sum_h}
In the case of \upps, $h = \{g^{\tau}\}^{\mathcal{T}}_{{\tau}=1} = \{\sum_{r \in \{r\}_{\tau}} \widetilde{\mathrm{H}}_{r}\}^{\mathcal{T}}_{{\tau}=1}$, where $g$ is the task-specific function from Equation \ref{eq:cond_p}, defined as the sum of the representations of $|\{r\}_{\tau}| = k$ rules/heads $\widetilde{\mathrm{H}}_{r}$, where $\{r\}_{\tau}$ is the set of chosen rule indexes for task $t$ defined in Equation \ref{eq:rules}. Since rule selection is dependent on task specific conditions, there are no guarantees that $g^{\widetilde{{\tau}}} = g^{\bar{{\tau}}}$ for any two tasks $\widetilde{{\tau}}, \bar{{\tau}} \in \{1,\dotsc,\mathcal{T}\}, \bar{{\tau}}\neq \widetilde{{\tau}}$.
\end{remark}

\begin{remark}
\label{rem:target_task}
Let $h(\bar{X_{\tau}}) \sim \mu_{\tau}$ define the representation distribution over $\mathbb{R}^d$ and ${\tau}=\widetilde{{\tau}}$ be the target task. A maximum positive transfer from tasks ${\tau} \in \{1,\dotsc,\mathcal{T}\}, {\tau}\neq \widetilde{{\tau}}$ to $\widetilde{{\tau}}$, implies that $\mu_{\widetilde{{\tau}}} = \mu_1 = \dotsc = \mu_{\mathcal{T}}$, even if $h$ may differ from task to task as stated in remark \ref{rem:sum_h}.
\end{remark}

\begin{corollary}
\label{cor:lower_gen_err}
Let $\widetilde{\mathcal{H}} \subseteq  \mathcal{H}$ and $\Delta(\widetilde{\mathcal{L}})$ be the multi-task learning generalization error when $\mu_{\widetilde{{\tau}}} = \mu_1 = \dotsc = \mu_{\mathcal{T}}$. With probability at least $1-\delta$, the bound $\mathcal{\widetilde{B}}_{\widetilde{\mathcal{H}}}$ is:
\begin{equation}
\label{eq:lower_gen_err}
\Delta(\widetilde{\mathcal{L}}) \leq \widetilde{\mathcal{B}}_{\widetilde{\mathcal{H}}} \leq \mathcal{B}_{\mathcal{H}}
\end{equation}%
\end{corollary}

\begin{corollary}
\label{cor:upps_gen_err}
Consider the \upps architecture where $\{r\}_{\tau}$ is a set of $k$ rules randomly chosen from $N$ options for task $t$ and that $\mu_{i} = \mu_{j}$ whenever $\{r\}_i=\{r\}_j$. With probability at least $1-\delta$, we find that the multi-task learning generalization error $\Delta(\mathcal{L_{\upps}})$ is bounded by:
\begin{multline}
\label{eq:upps_gen_err}
\Delta(\mathcal{L_{\upps}}) \leq
(1-\mathcal{P}(N,k,\mathcal{T})) \widetilde{\mathcal{B}}_{\widetilde{\mathcal{H}}} + \mathcal{P}(N,k,\mathcal{T}){\mathcal{B}}_{{\mathcal{H}}}
\end{multline}%
where $\mathcal{P}(N,k,\mathcal{T}) = \frac{(\mathcal{T}k)!\mathcal{S}(N,\mathcal{T}k)}{k^N}$ and $\mathcal{S}(N,\mathcal{T}k)$ is the Sterling Number of the second kind. Equations \ref{eq:upps_gen_err}, \ref{eq:lower_gen_err} imply:
\begin{equation}
\label{eq:upps_gen_bound}
\sup{\Delta(\widetilde{\mathcal{L}})} \leq \sup{\Delta(\mathcal{L_{\upps}})} \leq \sup{\Delta(\mathcal{L})}
\end{equation}%
\end{corollary}
An observation from Corollary \ref{cor:upps_gen_err} is that $\Delta(\mathcal{L_{\upps}})$ is a function of $k$ and $N$ via $\mathcal{P}(N,k,\mathcal{T})$. We refer the reader to Appendix \ref{sec:more_theory} for proofs and assumptions.

\subsection{Experiment: How do $k$ and $N$ affect performance?}
\label{app:k_N_ablation}
We show in Table \ref{tab:k_v_n} that using more rules (higher $k$) does not always translate into better performance. 
Further, using all rules is equivalent to our \trsfp and \utrsfp baselines.

\begin{table}[ht!]

    \footnotesize
    \centering
    \setlength\tabcolsep{4.5pt}
    \vskip 0.05in
    \begin{tabular}[b]{c|cccc|cccc}
    \toprule
     & \multicolumn{4}{c}{$1-\mathcal{P}(N,k,\mathcal{T})$} & \multicolumn{4}{c}{\upps} \\
    \diagbox[height=1.5\line]{$N$}{$\mathcal{T}k$}  & 4 & 6 & 8 & 10 & 4 & 6 & 8 & 10 \\
    \midrule
    12  & \gradient{0.13} & \gradient{0.56} & \gradient{0.91} & \gradient{0.99} & \gradientresult{23.6} & \gradientresult{27.7} & \gradientresult{27.8} & \gradientresult{27.2} \\
    13  & \gradient{0.09} & \gradient{0.49} & \gradient{0.86} & \gradient{0.73} & \gradientresult{22.8} & \gradientresult{27.4} & \gradientresult{28.3} & \gradientresult{30.2} \\
    14  & \gradient{0.07} & \gradient{0.42} & \gradient{0.81} & \gradient{0.97} & \gradientresult{22.5} & \gradientresult{27.1} & \gradientresult{29.0} & \gradientresult{27.6} \\
    15  & \gradient{0.05} & \gradient{0.36} & \gradient{0.75} & \gradient{0.95} & \gradientresult{22.0} & \gradientresult{26.8} & \gradientresult{29.3} & \gradientresult{27.6} \\
    16  & \gradient{0.04} & \gradient{0.30} & \gradient{0.69} & \gradient{0.93} & \gradientresult{21.8} & \gradientresult{26.8} & \gradientresult{30.7} & \gradientresult{27.8} \\
    \bottomrule
        
    \end{tabular}

    \caption{\label{tab:k_v_n} \small Theoretical vs. Experimental assessment of $k$ and $N$. For \upps, we present the average test score of $\protect\overrightarrow{\text{fr-en}}$ and $\protect\overrightarrow{\text{en-fr}}$.
    }
\end{table}

We assess \upps dependence on $k$, the top-$k$ selection of rules, and $N$, the number of heads/rules in the architecture, for $\mathcal{T}=2$ tasks. 
For this analysis, we choose Europarl language pairs $\overrightarrow{\text{fr-en}}$ and $\overrightarrow{\text{en-fr}}$, trained and tested on $1000$ samples. 
We see in Table \ref{tab:k_v_n}, for $\mathcal{T}k=\{4,6\}$, that higher $1-\mathcal{P}(N,k,\mathcal{T})$ does translate into higher \upps performance as predicted by Equation \ref{eq:upps_gen_bound}. 
However, for other $\mathcal{T}k$ values \upps seems to reach a maximum at $1-\mathcal{P}(N,k,\mathcal{T})\approx 0.7$, and then drops when $1-\mathcal{P}(N,k,\mathcal{T}) > 0.7$.  
While equation \ref{tab:k_v_n} does not predict such threshold, we conjecture that too much rule overlap may oversmooth task specific information, since equal $g^t$ across tasks may allow $\mu_1 = \dotsc = \mu_{\mathcal{T}}$  (see remarks \ref{rem:sum_h} and \ref{rem:target_task}). 
Nonetheless, the heat-map in Table \ref{tab:k_v_n} expresses a similar pattern between $1-\mathcal{P}(N,k,\mathcal{T})$ and sacreBLEU scores. Interestingly, we also note that $k=\frac{n}{2} > 1$ is typically the maximum performance for a given $n$, which is also the maximum number of combinations that we get from choosing $k$ rules from $n$ candidates. 
More generally, \upps works best when $\mathcal{T}k=\frac{1}{2}N$ (see Table \ref{tab:k_v_n}), for $\mathcal{T}$ the number of tasks and $N$ the total number of available rules. 
As $\mathcal{T}$ increases, \upps becomes more sparse since the number $k$ also decreases.
We show in Table \ref{tab:nmt_sep} that one rule can be associated with one task, creating a very sparse rule-task association similar to ProdSys.

\subsection{Are modules reused and composed for new tasks?}
\label{app:rule_compose}

To analyse the effects of a detailed task instruction on rule selection, we use the experiments in the previous Section \ref{sec:composition}. We measure head usage statistics on 1K examples from validation sets. 
The results plotted in Figure \ref{fig:rule_usage} show that (1) heads are sparsely selected, (2) there is little overlap between support task head usage\footnote{We show in Appendix~\ref{sec:sparse_selection} that under certain conditions, one rule can be directly and sparsely associated with one task.} and (3) the target task reuses the highest frequency heads of each support task. This results validates both sparse selection and module composition hypothesis.


\begin{figure}[ht]
\centerline{
\includegraphics[trim={0 0 0 0}, width=0.49\textwidth]{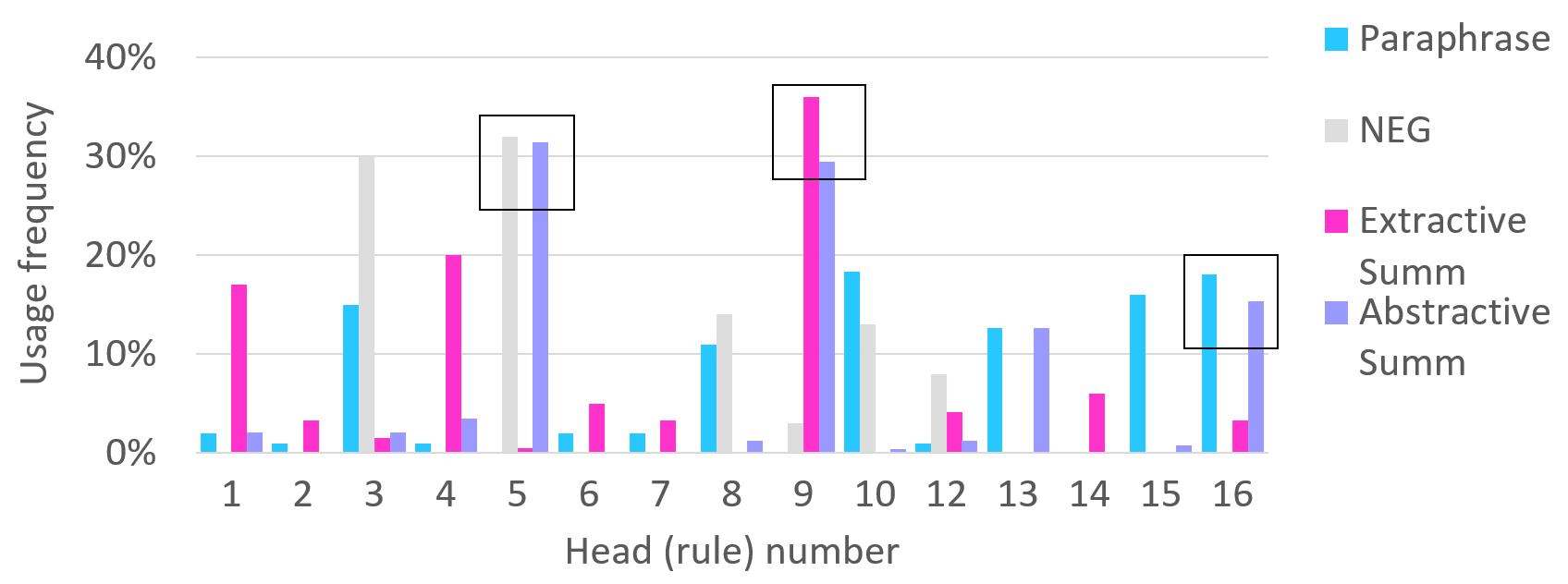}}
\vspace{-0.5em}
\caption{\small
\upps rule usage frequency of generated prompts for the first PLM layer. The three boxes correspond to the top 3 most selected heads for the target task (Abstractive Summarization) in a zero-shot setting. The boxes also correspond to the most frequently used head for the three support tasks (Paraphrase, NEG, Extractive Summarization) used to pretrain \upps.
}
\label{fig:rule_usage}
\end{figure}

\subsection{How does the number of bridge tasks affect composition?}
\label{app:bridge_compose}

\begin{wrapfigure}[9]{r}{0.15\textwidth}
\vspace{-15pt}
    \begin{center}
        \includegraphics[width=0.14\textwidth]{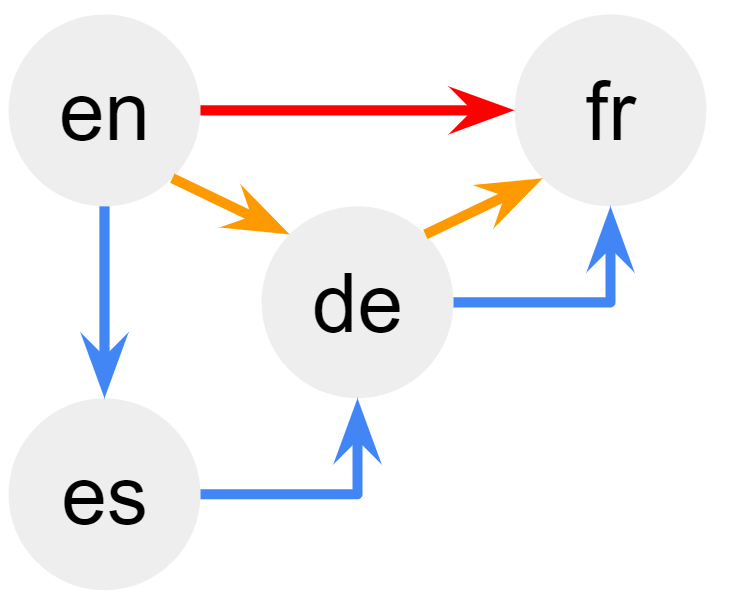}
    \end{center}
\vspace{-5pt}
    \caption{\small \label{fig:bridges} Bridges from en to fr.}
\end{wrapfigure}

As seen in Table~\ref{tab:conditions_eg}, translation was also a way to test ``task transfer''; e.g.: if we train with pairs en $\rightarrow$ fr and fr $\rightarrow$ de, the model can translate en $\rightarrow$ de with little supervision (see Section~\ref{sec:multitask} for more details).
The form of composition that we have experimented with is either induced from the choice or wording of textual conditions as explained in Section \ref{sec:composition} or expressed via the module selection mechanism explained in Appendix~\ref{app:mech_composition}.
We are also interested in understanding how ``bridge'' tasks (used as support tasks) can influence the performance of a target task when the path to solve a target task from bridge tasks is explicit.
For instance, in a multilingual setting (experimental results in Section \ref{sec:multitask}), \emph{choosing} to train \textbf{English (en) - French (fr) - German (de) - Spanish (es)} jointly offers multiple bridges as seen in Figure~\ref{fig:bridges}.
The mapping from en to fr can be learned directly (red line), by composing 2 tasks (orange line) or by composing 3 tasks (blue line).

\begin{table}[h]
    \footnotesize
    \centering
    \begin{tabular}[b]{l|cc}

        \textbf{Trained} & \multicolumn{2}{c}{Europarl \textbf{$\overrightarrow{\text{en-fr}}$}}\\
        \textbf{with:}    & \highlight{\textbf{\upps}} & \mbart \\
        \hline \\
        $\protect\overrightarrow{\text{en-fr}}$ & 13.4 & 13.6 \\
        + $\overrightarrow{\text{en-de}},\overrightarrow{\text{de-fr}}$ & 14.8 & 13.9 \\
        + $\overrightarrow{\text{en-es}},\overrightarrow{\text{es-de}},\overrightarrow{\text{de-fr}}$ & 15.4 & 14.0 \\

        \hline
    \end{tabular}
    \caption{\label{tab:bridge-tasks} \small Europarl translation test sacreBLEU ablation. $\protect\overrightarrow{\text{en-fr}}$ results. Each pair is trained with 50 examples chosen randomly.
    }
\end{table}

Our initial hypothesis was that if we trained models with either the two $\{\overrightarrow{\text{en-de}},\overrightarrow{\text{de-fr}}\}$ or the three $\{\overrightarrow{\text{en-es}},\overrightarrow{\text{es-de}},\overrightarrow{\text{de-fr}}\}$ unidirectional language pairs, we would be able to zero-shot translate $\overrightarrow{\text{en-fr}}$ using simple task instructions (e.g., "Translate English to French") and a single pass through the model.
However, all of our CPG models and our other baselines, including \mbart FT performed very poorly on this task, with SacreBLEU scores between 2 and 4.
We performed another ablation study in Table \ref{tab:bridge-tasks} to understand the effects of bridge tasks.
We first notice that \mbart FT has better performance in few-shot $\overrightarrow{\text{en-fr}}$ translation.
However, \upps's improvement from adding bridge tasks is much more pronounced with a +1.4 and +2.0 sacreBLEU score jump when adding $\{\overrightarrow{\text{en-de}},\overrightarrow{\text{de-fr}}\}$ or $\{\overrightarrow{\text{en-es}},\overrightarrow{\text{es-de}},\overrightarrow{\text{de-fr}}\}$ respectively to the initial $\overrightarrow{\text{en-fr}}$ training data. This experiment shows that \upps is better able to join and compose bridge tasks than the strong \mbart FT baseline.

\subsection{Are rules sparsely selected?}
\label{sec:sparse_selection}

As previously mentioned in Section \ref{sec:method}, \upps sparsely selects rules for a given situation or condition. In this analysis, we assess rule separation in an experiment similar to \citeauthor{goyal2021neuralprodsys} where our transformations are the translations $\overrightarrow{\text{en-xx}}$ from English to a target language xx\footnote{es=Spanish, da=Danish, el=Greek, de=German}. We train \upps with $20$k samples per language pair, $\mathcal{T}=4$ tasks, $k=4$ selections, $N=4$ heads and without using PLM layer hidden representations as conditions\footnote{We only use task instructions such as ``Translate fr to en''.}. Similarly to Neural Production Systems, \textbf{\upps successfully learns to represent each task with a separate rule} as presented in Table \ref{tab:nmt_sep}.

\begin{table}[ht!]

\footnotesize
\centering 
\renewcommand{\arraystretch}{1.2}

    \begin{tabular}{|c|c|c|c|c|}
        \hline
        & Rule 1 & Rule 2 &  Rule 3 & Rule 4 \\
        \hline
        $\overrightarrow{\text{en-es}}$ & 20k & 0 & 0 & 0 \\
        $\overrightarrow{\text{en-da}}$ & 0 & 20k & 0 & 0 \\
        $\overrightarrow{\text{en-el}}$ & 0 & 0 & 20k & 0 \\
        $\overrightarrow{\text{en-de}}$ & 0 & 0 & 0 & 20k \\
        \hline
    \end{tabular}

\caption{\label{tab:nmt_sep} \small Rule separation and sparse selection. Each cell indicates the number of times the corresponding rule is used for the given operation.}
\end{table}

\section{Additional details on general setup}
\label{app:more_gen_setup}

\paragraph{Prompt Length:} Since $T_C$ can be large, we specify a prompt length $T_P < T_C$ and take the first vectors $\langle \vc_t^{(L)} | t \in \{1, \ldots,  T_P\}\rangle$. The prefix length $T_C$ does not change the learnable parameters since conditional prompt generators (including \upps) are based on Transformer architectures that can produce variable size sequences.
This is another advantage compared to the baseline \prefix.

\paragraph{Datasets:} (1) XSum \cite{xsum} is an article to single-sentence abstractive summarization dataset. We include an out-of-topic (OOT) split \cite{prefix} where the model is trained on \{world, UK, business, news\} and tested on the remaining unseen categories such as politics, technology. (2) Topic-CNN-DM \cite{topic-aware} is a controlled summarization task that is guided by a specific subject. (3) Multilingual translation using Europarl \cite{europarl}. This dataset is used to evaluate multi-task abilities when we jointly train two and five language pairs.
(4) Semantic Parsing dataset SCAN \cite{scan} and seven compositional generalization splits where models generate a sequence of actions based on a command (see example \ref{ex:ex_scan}). 
XSum has 204.0K and 11.3K training and testing examples. 
XSum-OOT, which has 128.7K and 20.4K training and testing examples. 
Topic-CNN-DM, our controlled generation task, has 74.8K and 31.8K training and testing examples respectively.

For zero-shot and few-shot experiments in Section~\ref{sec:composition}, when the target task is XSum abstractive summarization, we pretrain our prompt generators with the following: 1) XSum Named Entity Generation (NEG) is a task similar to Named Entity Recognition. However we input a label (e.g. organization, country, person) and ask the model to generate entities as they appear in the article. 2) XSum extractive summarization, where the model is required to generate a sentence from the input article that is closest to the ground-truth summary. 3) XSum paraphrase tasks the model to generate novel formulations of the input sentences. We created all three datasets. Similarly, for Topic-CNN-DM, our three intermediate tasks are: 1) a CNLG version of question-answering dataset newsQA \cite{newsqa}, 2) CNN-DM extractive summarization \cite{cheng-lapata-2016-neural} and 3) XSum abstractive summarization. We refer the reader to Appendix \ref{sec:datasets} for more details on the data creation steps for XSum NEG/paraphrase, and data urls.

\paragraph{Training and Evaluation:} Pretrained models were loaded from HuggingFace.
Unless otherwise specified, we train \trsfp, \utrsfp and \upps with learning rate $1e^{-3}$, a seed of $1$, a length penalty of $1$ and an effective batch size of $96$, a warm-up ratio of $0.05$, use ``simple'' task instructions (Table \ref{tab:task_description}). 
We add prompt vectors to the bottom and top two layers for both encoder and decoder. For Topic-CNN-DM, SCAN and Europarl experiments, we run CPG models for $10$, $5$ and $20$ epochs respectively and \upps had $n=16$ rules and top $k=3$. For XSum, we train for $20$ epochs and use $n=12$. 
We report ROUGE-1/2/L 
and Meteor 
scores for abstractive summarization, sacreBLEU 
for translation and accuracy for semantic parsing. For \upps in Appendix~\ref{app:more_gen_setup}, we tried $k=\{2,3,4\}$ and $n=\{12,16\}$ on summarization, translation and semantic parsing. The percentage of trained parameter is calculated by dividing the trainable prompt generator parameters by all parameters (including the untrained PLM layers).

For zero-shot and few-shot experiments in Section~\ref{sec:composition}, we use two random seeds and train each support task jointly for $10$ epochs. Our target task is trained for 200 epochs with $\{50, 100, 200, 500\}$ training examples and tested on the models with the highest validation score.



\section{Additional details on datasets}
\label{sec:datasets}

\textbf{Xsum NEG} was constructed using Python library spaCy.
For each XSum dataset input article, we first detected all Named Entities (NE) using spaCy and kept the top 10 most frequent entities with the corresponding category. The NEG tasks consists of the following: given an input article and an NE category (e.g. “Companies, agencies, institutions” or “Countries, cities, states”) we need to generate the most common NE in the article (e.g. “Sony” or “UK”).

\textbf{Xsum Extractive} is created by scoring each sentence in the source article with the reference abstractive summary via F1 of ROUGE-1 and ROUGE-2, and use the highest scored one as the extractive target. 
For \textbf{Xsum Paraphrase}, we use Amazon Translate \footnote{https://aws.amazon.com/translate/} to perform round-trip translation from English to several pivot languages, and then back to English. Through this method, we aim to obtain paraphrases that are semantically equivalent with more lexical variations, thus encouraging the target abstractive summarization task to generate more novel text. Our pilot study (as in example \ref{ex:pivot}) shows that including Arabic and Chinese in the pivot languages \emph{generates more lexical variation compared to using European languages alone}.

\begin{example} \label{ex:pivot} 
\textbf{Source:} Those who participated in the Aberdeen Children of the 1950s project, 
which saw all primary pupils aged seven to 12 surveyed by the Medical Research Council in 1962, have been contacted. \myul[green]{They have been asked} to take part in the Scottish Family Health Study, which is recruiting 50,000 people. 
\myul[orange]{It aims to investigate} why diseases such as cancer can run in families. 
Those recruited will have their health tracked, \myul[yellow]{with the intention of} creating a Scottish "bio-bank" containing genetic, medical and family history and lifestyle information. 
This will allow researchers to investigate the contribution of genetic and environmental factors to common conditions. Blair Smith, a GP and professor of primary care medicine at Aberdeen University, will run the project.

\textbf{en-fr-de-en:} Those who participated in the Aberdeen Children of the 1950 project, 
where all elementary students between seven and twelve years of age were interviewed by the Medical Research Council in 1962, were contacted. 
\myul[green]{They were asked} to participate in the Scottish Family Health Study, which recruits 50,000 people. 
\myul[orange]{The aim is to investigate} why diseases such as cancer can occur in families. 
Recruited individuals see their health status being monitored \myul[yellow]{with the intention of} creating a Scottish “biobank” that contains genetic, medical, and family history and lifestyle information. 
Contribution of genetic and environmental factors to the general conditions Blair Smith, general practitioner and professor of basic care medicine at the University of Aberdeen, will lead the project.

\textbf{en-ar-fr-zh-en:} In 1962, the Medical Research Board contacted those who participated in the 1950s Aberdeen Children's Project, which saw all primary school students between the ages of 7 and 12. 
\myul[green]{They were invited} to take part in the Scottish Family Health Study, which recruited 50,000 people. 
\myul[orange]{Its purpose is to study} the causes of diseases such as cancer that may occur in the family. 
The health of recruiters will be tracked \myul[yellow]{in order to} create a Scottish “biobank” containing information on genetics, medicine and family history, as well as lifestyle. 
This will enable researchers to study the contribution of genetic and environmental factors to common conditions. Blair Smith, General Doctor and Junior Medicine Professor at the University of Aberdeen, will lead the project.

\end{example}

\begin{table}[ht!]
\footnotesize
\centering 
\renewcommand{\arraystretch}{1.2}

\vspace{2em}
    \begin{tabular}{|l|c|c|c|}
        \hline
        Dataset & Train &  Val & Test  \\
        \hline
        Xsum NEG        & 403,766 & 21,267 & 22,659 \\
        Xsum Extractive & 489,799 & 25,780 & - \\
        Xsum Paraphrase & 215,344 & 11,335 & - \\
        Xsum Abstractive & 204,017 & 11,333 & 11,327\\
        News QA & 2,992,677 & 82,821 & 127,261\\
        CNN-DM Extractive & 1,460,785 & 69,151 & 59,023\\
        Topic-CNN-DM     & 92,550 & 5,167& 5,127\\
        \hline
    \end{tabular}

\caption{\label{tab:dataset} Details of data sets. Some data sets used in intermediate tasks only do not have test split.}
\end{table}

News-QA\footnote{https://huggingface.co/datasets/newsqa}, CNN-DM Extractive Summarization\footnote{https://www.microsoft.com/en-us/research/project/newsqa-dataset/download/}, SCAN\footnote{https://huggingface.co/datasets/scan} and Europarl\footnote{https://huggingface.co/datasets/europarl\_bilingual} are publicly available datasets. The dataset sizes for each Europarl language pair can be found at this \href{https://opus.nlpl.eu/Europarl.php}{link}. For Europarl, the training set includes 90\% of the data, randomly chosen, and the remaining 10\% are used for the test and evaluation sets. For SCAN, the dataset sizes for each split can be found \href{https://github.com/brendenlake/SCAN}{here}.

\section{Additional details on baselines}
\label{app:baselines}

All PLMs are ``Large'' models with 12 Transformer layers in the encoder and 12 layers in the decoder. This translates into 406M, 610M and 770M parameters for \bart \cite{bart}, \mbart \cite{mbart}, \tfive \cite{t5} respectively. 
As we noted previously and shown in Figure~\ref{fig:baselines}, , \utrsfp is the same as \upps without the rule selection mechanism. 
In contrast to \trsfp and \prefix, the same \utrsfp and \upps layers are reused to generate prompts for all PLM layers. 
\trsfp and \prefix uses three different sets of parameters for the encoder, the decoder and the cross frozen PLM attention. 
For CPGs conditions are first tokenized and passed through a pretrained word embedding. We remind the reader that \upps is type of CPG.

\prefix introduces two hyperparameters that are important to performance: 1) the length of the prefix and 2) the number of layers of the parametrization network. Unlike prior work however, we provide theoretical guidelines that are generally applicable to a practitioner’s use case in Table~\ref{tab:k_v_n} of Appendix~\ref{app:k_N_ablation}. The theoretical and empirical study of k should simplify much of the burden of searching for the best k. Please note also, that we keep our seed equal to one in all high-resource experiments. This should be counted as one less hyperparameter to tune.

Note that \upps, \utrsfp \trsfp use multiheads as transformation functions. However, \utrsfp and \trsfp concatenate \emph{all} head representations and do not have a rule selection mechanism. As specified in \citeauthor{t5} when detailing the experiment setup of \adapter on \tfive, ``an inner dimensionality'' of 2048 achieves the best results, particularly since ``higher resource tasks require a large dimensionality to achieve reasonable performance''. We use 2048, or 14\% of trained parameters, to train \adapter.

\begin{figure}[ht]
\centerline{
\includegraphics[trim={0cm 0 0.1cm 0}, width=0.49\textwidth]{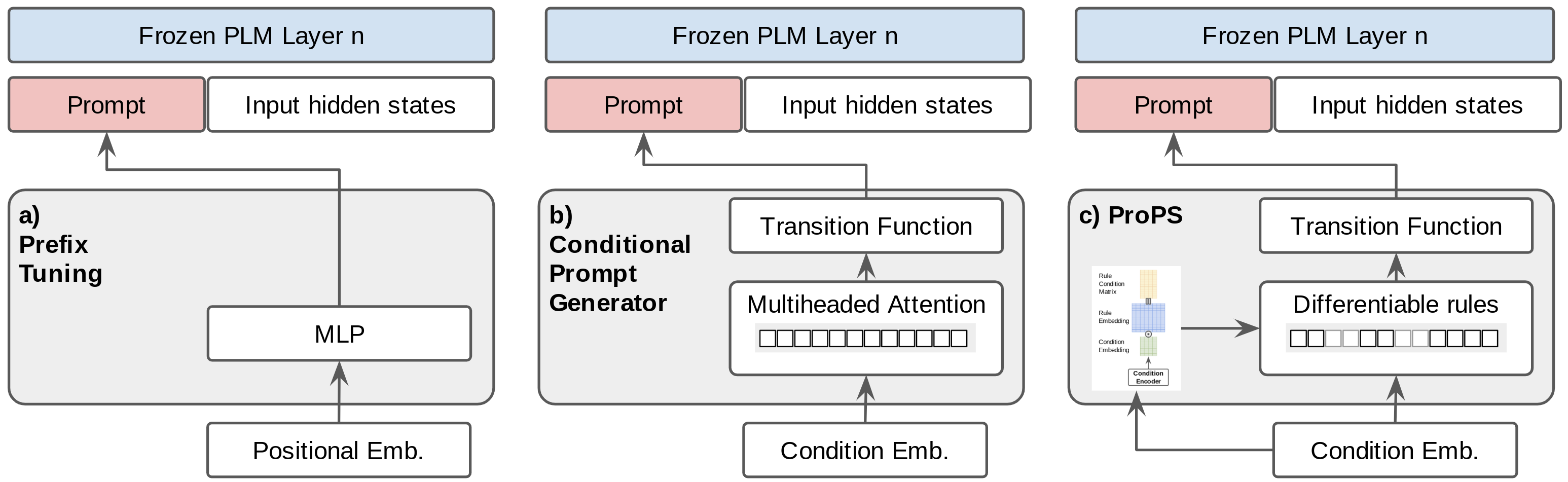}}
\caption{
\small Comparing architectures and visualizing baselines. On the left (a): the \prefix \protect\cite{prefix} Unconditional Prompt Generator uses a different model for encoder, decoder and inter frozen PLM attention. In the middle (b): a CPG with a single Transformer layer which can shared for encoder, decoder and inter frozen PLM attention. On the right (c): \upps.}
\label{fig:baselines}
\end{figure}

In Figure~\ref{fig:adapters}, we show the architectures of our adapter baselines. \camtl \cite{pilault2021conditionally} is a Transformer based adapter hypernetwork that works well in a NLU multi-task setting. For multilingual experiments in Section \ref{sec:main_results}, \camtl layers are insert into \mbart and uses a different task embedding for each language pair and translation direction.

\begin{figure}[ht]
\centerline{
\includegraphics[trim={0cm 0 0.1cm 0}, width=0.49\textwidth]{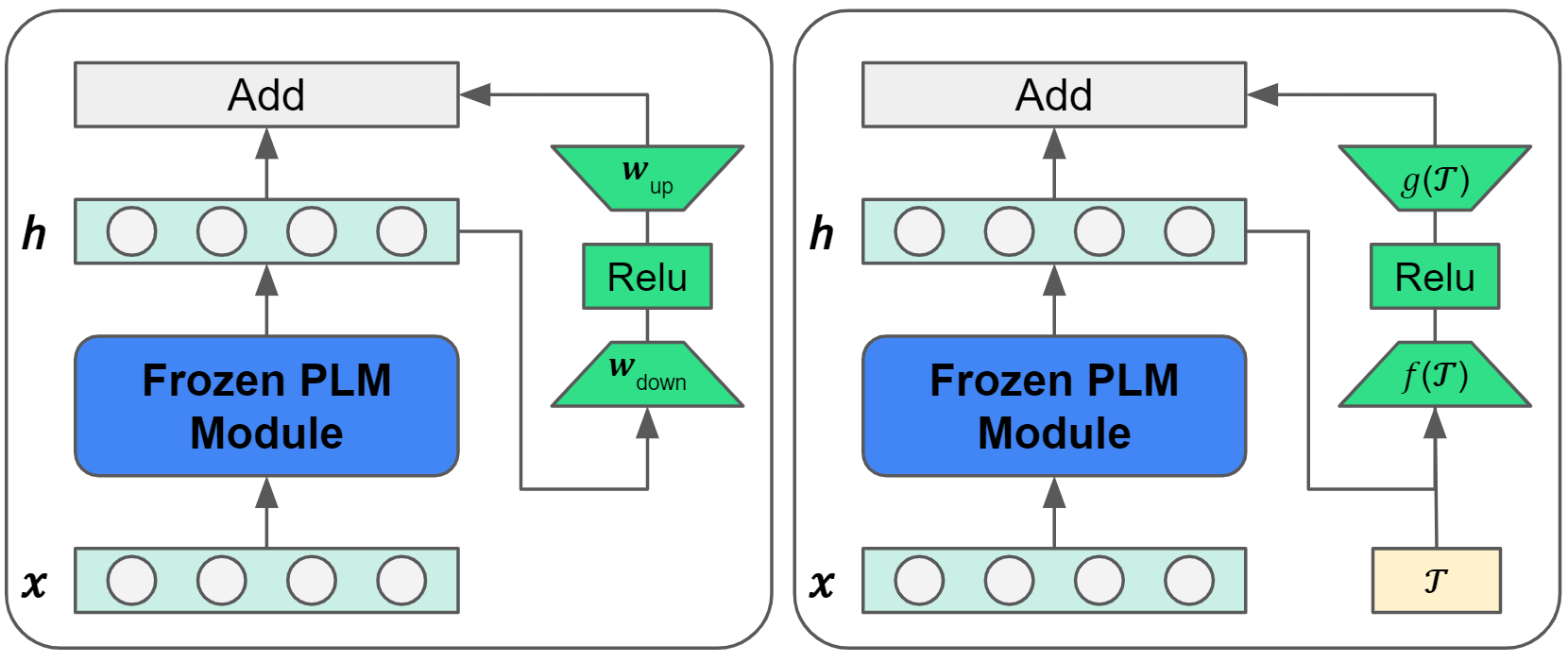}}
\vspace{-0.5em}
\caption{
\small A simplified graphical illustration of adapter architectures used as baselines. On the right the single task \adapter from \protect\citeauthor{adapter} and on the left \camtl, a multi-task adapter. \camtl uses hypernetwork, i.e., weights are generated by functions $f(\cdot)$ and $g(\cdot)$ conditioned on task embeddings $\mathcal{T}$.}
\label{fig:adapters}
\end{figure}

Note that unlike the Transformer multi-headed attention, $\widetilde{\mathrm{H}}_{r}$ heads (rules) are not concatenated but summed. 
The concatenation of head outputs does not maintain order and permutation invariance if $k < n$ heads are selected, potentially explaining much lower performance when compared to summation. 
Also note that $\widetilde{\mathrm{H}}_{r}$ does not attend to all condition sequences but only to itself $h_t$ and the selected context sequences $\mathcal{S}_{C'}$.

\section{Proofs and Assumptions}
\label{sec:more_theory}

Theorem \ref{thm:gen_err} presented a task-averaged empirical generalization bound for multi-task learning problems defined by Equation \ref{eq:theorem}.  We begin with the proof of Corollary \ref{cor:lower_gen_err} below:

\begin{proof}
Assume that for all $\mathcal{T}$ tasks, we have $\mu_{\widetilde{\tau}} = \mu_1 = \dotsc = \mu_\mathcal{T}$. With probability at least $1-\delta$, for every $\widetilde{h} \in \widetilde{\mathcal{H}}$, the Gaussian average $G(\widetilde{\mathcal{H}}(\bar{X}))$ becomes:

\begin{multline}
G(\widetilde{\mathcal{H}}(\bar{X})) 
= \mathbb{E}\sup_{\widetilde{h} \in \widetilde{\mathcal{H}}}\sum_{i,\tau}\gamma _{i}\widetilde{h}_{i}(x_{\tau}) = \\
\mathcal{T} \cdot \mathbb{E}\sup_{\widetilde{h} \in \widetilde{\mathcal{H}}}\sum_{i}\gamma _{i}\widetilde{h}_{i}(x_{\widetilde{{\tau}}}) = \mathcal{T} \cdot G(\widetilde{\mathcal{H}}(\bar{X}_{\widetilde{{\tau}}})) \label{eq:gaussian_tilde}
\end{multline}

Since $\mu_{\widetilde{{\tau}}} = \mu_1 = \dotsc = \mu_{\mathcal{T}}$, $\widetilde{\mathcal{H}} \subseteq  \mathcal{H}$, where $\mathcal{H}$ is defined in Section \ref{app:theory}. Therefore, from the property of Gaussian averages we have $T \cdot G(\widetilde{\mathcal{H}}(\bar{X}_{\widetilde{{\tau}}})) \leq G(\mathcal{H}(\bar{X}))$. Similarly, we have:

\begin{multline}
\sup_{\widetilde{h} \in \widetilde{\mathcal{H}}} \Vert \widetilde{h}( \bar{X}) \Vert 
= \sup_{\widetilde{h} \in \widetilde{\mathcal{H}}}\sqrt{\sum_{i,t}\widetilde{h}_{i}(x_{{\tau}})^2} = \sqrt{T} \cdot \sqrt{\sum_{i}\widetilde{h}_{i}(x_{\widetilde{{\tau}}})^2} \\
= \sqrt{\mathcal{T}} \cdot \sup_{\widetilde{h} \in \widetilde{\mathcal{H}}} \Vert \widetilde{h}( \bar{X_{\widetilde{{\tau}}}}) \Vert \leq \sup_{h \in \mathcal{H}} \Vert h( \bar{X}) \Vert \label{eq:h_norm}
\end{multline}

It follows, from equations \ref{eq:gaussian_tilde} and \ref{eq:h_norm}, that:
\begin{multline}
\Delta(\widetilde{\mathcal{L}}) \leq \widetilde{\mathcal{B}}_{\widetilde{\mathcal{H}}} = \\
c_{1}\frac{\rho G(\widetilde{\mathcal{H}}(\bar{X}))}{n\mathcal{T}}+c_{2}\frac{%
Q(\mathcal{F})\sup_{\widetilde{h}\in \widetilde{\mathcal{H}}}\left\Vert \widetilde{h}\left( \bar{X}\right) \right\Vert}{n\sqrt{\mathcal{T}}} + c_{3} \leq \\
c_{1}\frac{\rho G(\mathcal{H}(\bar{X}))}{n\mathcal{T}}+c_{2}\frac{%
Q(\mathcal{F})\sup_{h\in \mathcal{H}}\left\Vert h\left( \bar{X}\right) \right\Vert}{n\sqrt{\mathcal{T}}} + c_{3} = 
\mathcal{B}_{\mathcal{H}}
\label{eq:proof1}
\end{multline}

\end{proof}

Before going over proof of Corollary \ref{cor:upps_gen_err}, we first state the following proposition:

\begin{proposition} 
\label{prop:selec}
If \upps uses a sparse selection of $k$ heads (rules) from $n$ possible choices, then there is a chance that tasks have complete overlap between the chosen rule indices $\{r\}_{\tau}$ from equation \ref{eq:rules}. In this instance, two tasks $\widetilde{{\tau}}, \bar{{\tau}} \in \{1,\dotsc,\mathcal{T}\}, \bar{{\tau}}\neq \widetilde{{\tau}}$ with overlapping rule indices $\{r\}_{\widetilde{{\tau}}}=\{r\}_{\bar{{\tau}}}$ will have $\mu_{\widetilde{{\tau}}} = \mu_{\bar{{\tau}}}$.
\end{proposition}

We also state and prove the following lemma:

\begin{lemma}
\label{lem:prob}
For $\mathcal{T}$ tasks, the probability of selecting with replacement $\mathcal{T}k$ unique rules from $N$ possible choice is:
\begin{equation}
    \mathcal{P}(N,k,\mathcal{T}) = \frac{(\mathcal{T}k)!\mathcal{S}(N,\mathcal{T}k)}{k^N},
\end{equation}
 where $\mathcal{S}(N,\mathcal{T}k)$ is the Sterling Number of the second kind.
\end{lemma}

\begin{proof}
    Each sample is equivalent to an $N$-vector with coefficients in a set of $k$ choices. Since there are $k$ ways to specify each coefficient, the total number of samples is $k^N$. Further, from \citeauthor{wagner2020first}, the number $\sigma$ of unique $\mathcal{T}k$ rules from $n$ possible choices is a surjection enumeration problem\footnote{A surjective function $f:A\mapsto B$, when $A$ and $B$ are any sets with $|A|=n$ and $|B|=\mathcal{T}k$.}, where $\sigma=(\mathcal{T}k)!\mathcal{S}(N,\mathcal{T}k)$. 
    
    It follows that the probability of selecting with replacement $\mathcal{T}k$ rules at least once from $N$ possible choice is $\frac{(\mathcal{T}k)!\mathcal{S}(N,\mathcal{T}k)}{k^N}$.
\end{proof}

If $\{r\}_{\tau}$ is a set of $k$ rules randomly chosen from $N$ options, our proof of Corollary \ref{cor:upps_gen_err} follows:

\begin{proof}
From Corollary \ref{cor:lower_gen_err}, we know that if \upps selection mechanism chooses the \emph{same} $k$ rules for all $\mathcal{T}$ tasks then $\mu_1 = \dotsc = \mu_{\mathcal{T}}$ and the generalization bound becomes $\widetilde{\mathcal{B}}_{\widetilde{\mathcal{H}}}$. However, from Lemma \ref{lem:prob}, we can infer that the probability of choosing overlapping rules (not unique) is $1 - \mathcal{P}(N,k,\mathcal{T})$. Therefore, the generalization bound of \upps $\mathcal{B}^{u}_{\mathcal{H}^u}$ is $\widetilde{\mathcal{B}}_{\widetilde{\mathcal{H}}}$ with probability $1 - \mathcal{P}(N,k,\mathcal{T})$ and $\mathcal{B}_{\mathcal{H}}$ with probability $\mathcal{P}(N,k,\mathcal{T})$. It follows that:
\begin{equation}
    \mathbb{E}[\mathcal{B}^{u}_{\mathcal{H}^u}]=(1-\mathcal{P}(N,k,\mathcal{T})) \widetilde{\mathcal{B}}_{\widetilde{\mathcal{H}}} + \mathcal{P}(N,k,\mathcal{T}){\mathcal{B}}_{{\mathcal{H}}}
\end{equation}
\end{proof}

\end{document}